\documentclass{article}
\usepackage{ijcai23}

\usepackage{times}
\usepackage{soul}
\usepackage{url}
\usepackage[hidelinks]{hyperref}
\usepackage[utf8]{inputenc}
\usepackage[small]{caption}
\usepackage{graphicx}
\usepackage{amsmath}
\usepackage{amsthm}
\usepackage{booktabs}
\usepackage{algorithm}
\usepackage{algorithmic}
\usepackage[switch]{lineno}

\usepackage{subfigure}
\usepackage{tikz}

\usepackage{amssymb}
\usepackage{multirow}

\theoremstyle{plain}

\newtheorem*{theorem*}{Theorem}
\newtheorem{proposition}{Proposition}
\newtheorem*{proposition*}{Proposition}

\theoremstyle{definition}

\newtheorem{assumption}{Assumption}

\title{Causal Deep Reinforcement Learning Using Observational Data}

\author{
Wenxuan Zhu$^1$
\and
Chao Yu$^2$\footnote{Corresponding author}
\And
Qiang Zhang$^1$\footnotemark[1]
\affiliations
$^1$Dalian University of Technology\\
$^2$Sun Yat-sen University
\emails
zhuwenxuan@mail.dlut.edu.cn,
yuchao3@mail.sysu.edu.cn,
zhangq@dlut.edu.cn
}

\DeclareMathOperator{\offText}{off}
\DeclareMathOperator{\PText}{P}
\DeclareMathOperator{\PrText}{Pr}
\DeclareMathOperator{\argmaxText}{argmax}
\DeclareMathOperator{\BetaText}{Beta}
\DeclareMathOperator{\NText}{N}

\begin{document}

\maketitle

\begin{abstract}
Deep reinforcement learning (DRL) requires the collection of interventional data, which is sometimes expensive and even unethical in the real world, such as in the autonomous driving and the medical field. Offline reinforcement learning promises to alleviate this issue by exploiting the vast amount of observational data available in the real world. However, observational data may mislead the learning agent to undesirable outcomes if the behavior policy that generates the data depends on unobserved random variables (i.e., confounders). 
In this paper, we propose two deconfounding methods in DRL to address this problem. The methods first calculate the importance degree of different samples based on the causal inference technique, and then adjust the impact of different samples on the loss function by reweighting or resampling the offline dataset to ensure its unbiasedness.
These deconfounding methods can be flexibly combined with existing model-free DRL algorithms such as soft actor-critic and deep Q-learning, provided that a weak condition can be satisfied by the loss functions of these algorithms. We prove the effectiveness of our deconfounding methods and validate them experimentally.
\end{abstract}

\section{Introduction}
Human beings can learn from observation (e.g., in astronomy) and experimentation (e.g., in physics). For example, people understand the laws of astronomy by observing the movements of celestial bodies and the laws of physics by doing physical experiments. Analogously, agents can learn in these two ways as well. In some cases, however, experimentation (i.e., the collection of interventional data) can be expensive and even unethical, while observational data are easy to obtain. For example, it is unsafe for an agent to learn to drive a car on a real road, but we can easily collect data from a human driving a car with sensors.

Reinforcement learning (RL)~\cite{sutton2018reinforcement} is generally regarded as an interactive learning process, which means that the agent learns from the interventional data generated by experimentation. As a type of RL methods, offline RL~\cite{cql,tutoriallevine,siegel2020keep,ernst2005tree,fujimoto2019off,kumar2019stabilizing,agarwal2020optimistic,jaques2019way} has been proposed to study how to enable RL algorithms to learn strategies from observational data without interacting with the environment.

Most existing offline RL methods are based on an assumption that $\mathcal{O}_1=\mathcal{O}_2=\mathcal{O}_3$, where $o_1{\in}\mathcal{O}_1$ denotes the observation of the environment in which the offline data is collected, $o_2{\in}\mathcal{O}_2$ denotes the observation in the offline data, and $o_3\in\mathcal{O}_3$ denotes the observation in the online data where the agent trained using the offline data is tested
(In the driving car example, $o_1$ represents the environmental information perceived by the driver, $o_2$ represents the environmental information collected by the sensors, and $o_3$ represents the environmental information perceived by the agent in the testing environment.).
This assumption, however, is often difficult to hold in real-world problems.
The driver may be able to see more broadly than the sensors, to judge if the road is slippery based on weather conditions, or even to receive traffic conditions around them based on the radio, therefore, it is generally that $\mathcal{O}_1 \neq \mathcal{O}_2$.
If the driver makes a decision based on the information not collected by the sensors, i.e., there are unobserved random variables (confounders) that affect the action and the next sensory observation at the same time, the observational data generated by the driver may be misleading. An agent then learns the wrong dynamics of the environment based on this misleading information, leading to a biased estimation of the value functions and the final policies. 

Some researches~\cite{sen2017identifying,kallus2018policy,wang2020provably,gasse2021causal} in recent years study how to train an agent using offline data with the confounders based on the causal inference techniques.
However, most of these researches~\cite{sen2017identifying,kallus2018policy} only target at the bandit problems.
Other approaches propose deconfounding methods in RL settings but they are based on certain assumptions such as linear reward/transition functions~\cite{wang2020provably}, small state spaces~\cite{gasse2021causal}, or a complex correlation between the features~\cite{wang2020provably}. Therefore, these approaches only work for specific types of offline RL algorithms, and cannot be applied to large and continuous environments (See Section~\ref{section:related_work} for more details.). 
To address these problems, we propose two kinds of deconfounding methods based on the importance sampling and causal inference techniques. 
Our deconfounding methods first estimate a conditional distribution density ratio through the least-squares conditional density estimation (LSCDE)~\cite{sugiyama2010Conditional,rothfuss2019conditional}, and then adjust the impact of different samples on the loss function to ensure the unbiasedness of the loss function.
Specifically, we make the following contributions.

\begin{itemize}
  \item [1)]
    Unlike the existing importance sampling techniques in off-policy evaluation (OPE)~\cite{kallus2020confounding,gelada2019off,hallak2017consistent}, we estimate a conditional distribution density ratio, which keeps constant in the learning process, and thus can be applied to the RL field.
  \item [2)]
    In the proposed deconfounding methods, we decouple the deconfounding process from the RL algorithm, i.e., we can uniformly incorporate the conditional distribution density ratio into the loss functions of the offline RL algorithms. In other words, our plug-in deconfounding methods can be combined with existing offline RL algorithms provided that a weak condition is satisfied.
  \item [3)]
    Furthermore, since we do not learn a latent-based transition model, and do not assume a complex correlation between the features, our deconfounding methods do not contain a complex computational process, and thus can be applied to large and continuous environments. 
  \item [4)]
    We prove theoretically that these two deconfounding methods can construct an unbiased loss function w.r.t. the online data, and thus improve the performance of the offline RL algorithms. The experimental results verify that the proposed deconfounding methods are effective: offline RL algorithms using deconfounding methods perform better on datasets with the confounders.
\end{itemize}

\section{Background}
\label{background}
In this section, we define a confounded Markov decision process (CMDP) (similar to~\cite{wang2020provably}) and two corresponding structural causal models (SCMs) to describe the RL tasks, in which the offline data include unobserved confounders between the action, reward and next state, and the online data include no confounder between these variables.
In specific, there are both continuous random variables $\boldsymbol{X}$ and discrete random variables $\boldsymbol{Y}$ in our problems. Without loss of generality, we assume that any discrete random variable belongs to the set of integers. The probability density function $\PText\left(\boldsymbol{x},\boldsymbol{y}\right)$ of the discrete and continuous random variables is given as follows:
\begin{equation}
    \label{eqn:pdf}
    \begin{split}
    \PText\left(\boldsymbol{x},\boldsymbol{y}\right)=\frac{\partial^{p}}{\partial x_1 \cdots \partial x_p}\PrText\left(\boldsymbol{X}\leq\boldsymbol{x},\boldsymbol{Y}=\boldsymbol{y}\right),
    \end{split}
\end{equation}
where $\boldsymbol{X} \in \mathbb{R}^p$ and $\boldsymbol{Y} \in \mathbb{Z}^q$.
The CMDP for unobserved confounders can be denoted by a nine-tuple $\left\langle\mathcal{S},\mathcal{M},\mathcal{A},\mathcal{W},\mathcal{R},P_1,P_2,P_3,\mu_0\right\rangle$, where $\mathcal{S}$ denotes the state space, $\mathcal{M}$ denotes the intermediate state space, $\mathcal{A}$ denotes the discrete action space, $\mathcal{W}$ denotes the confounder space, $\mathcal{R}$ denotes the reward space, $P_1\left(s',r|s,w,m\right)$ denotes the dynamics of the CMDP, $P_2\left(w|s\right)$ denotes the confounder transition distribution, $P_3\left(m|s,a\right)$ denotes the intermediate state transition distribution, and $\mu_0\left(s\right)$ denotes the initial state distribution.
Note that the environment generates the intermediate state $m$ in the process of generating the next state $s'$.
For example, a physician prescribes some drugs to a patient. The actual amount of drugs taken by the patient (i.e., $m$) may be different from that prescribed by the physician (i.e., $a$) due to compliance issues. 
Also note that we integrate the state transition distribution and the reward transition distribution into the dynamics of the CMDP for notational convenience.

\begin{figure}[t]
    \centering
    \subfigure{
        \includegraphics[width=1.55in]{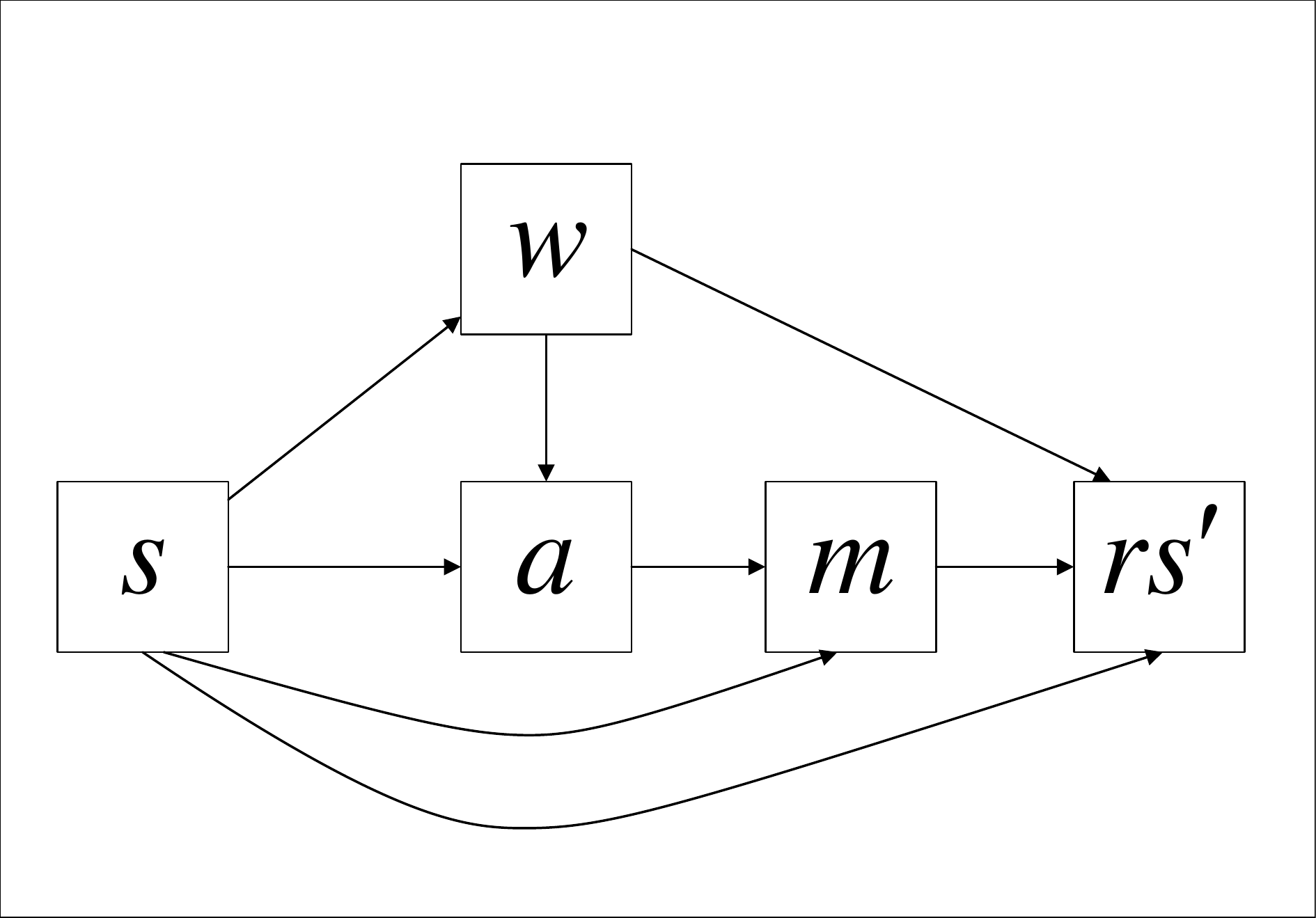}
        \label{cmdp2-scm_off}
    }
    \subfigure{
	\includegraphics[width=1.55in]{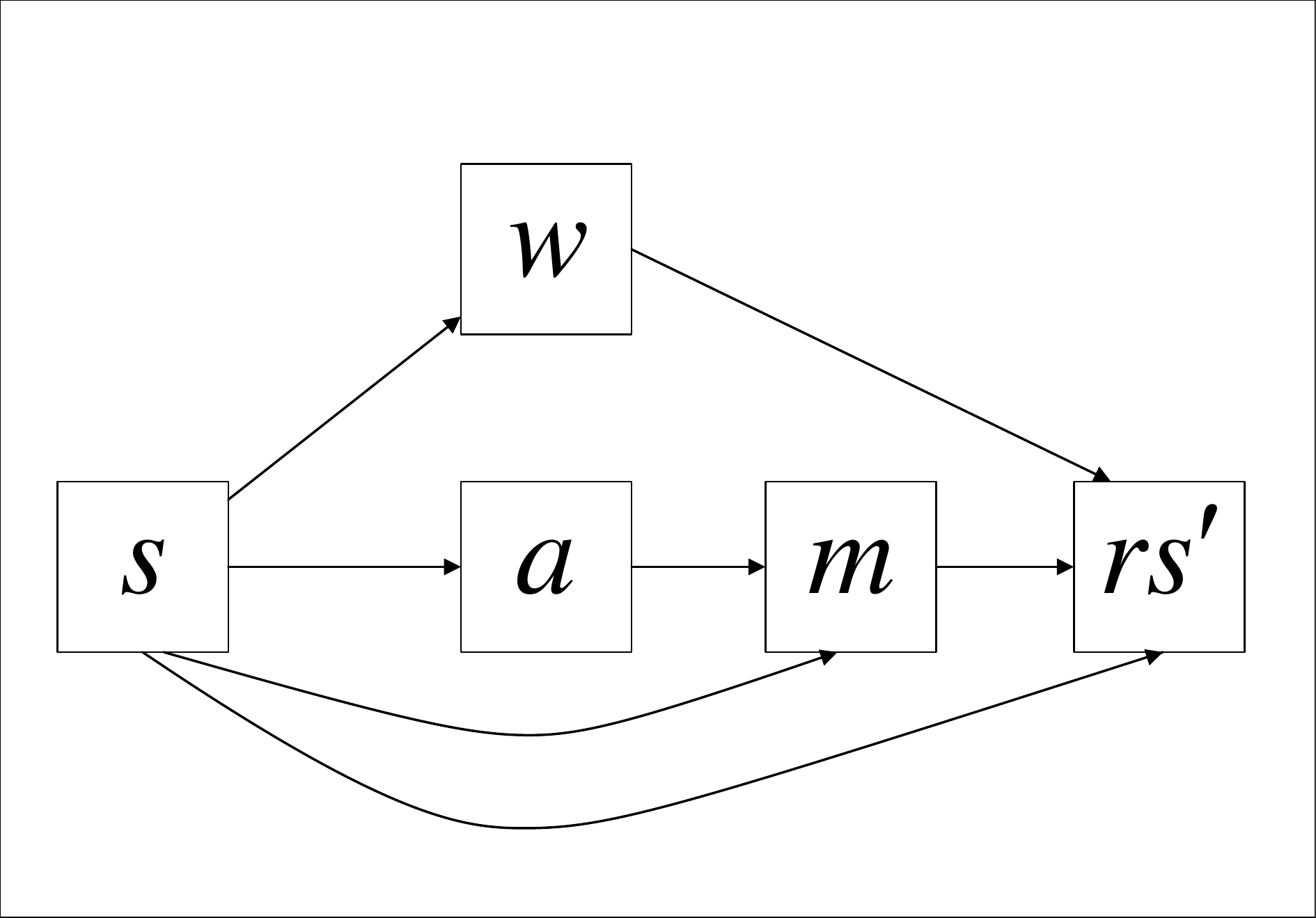}
        \label{cmdp2-scm_on}
    }
    \caption{The left and right subfigures represent the SCM in the offline setting and online setting, respectively, which correspond to the CMDP for unobserved confounders. The behavior policy depends on $w$ in the offline setting, but not in the online setting.}
    \label{figure:scms_unobserved}
\end{figure}

The SCM in the offline setting as shown in the first column of Figure~\ref{figure:scms_unobserved} can be defined as a four-tuple $\left\langle U, V, F, P_e\right\rangle$, where $U$ is the exogenous variables, which are not visible in the experiment, $V$ is the endogenous variables, including $\left(s, m, a, w, s', r\right)$, $F$ is the set of structural functions, including the state-reward transition distribution $P_1\left(s',r|s,w,m\right)$, the confounder transition distribution $P_2\left(w|s\right)$, the intermediate state transition distribution $P_3\left(m|s,a\right)$, and the behavior policy $\pi_b\left(a|s,w\right)$, and $P_e$ is the distribution of exogenous variables. 
The positivity assumption here is that, for $m \in \mathcal{M}, a \in \mathcal{A}, s \in \mathcal{S},w \in \mathcal{W}$ such that $\PText(s,w)>0$, $\PText(m,a|s,w)>0$.
The SCM in the online setting where the agent can intervene on the variable $a$ as shown in the second column of Figure~\ref{figure:scms_unobserved} can be defined as another four-tuple $\left(U, V, F, P_e\right)$, where the set of structural functions $F$ includes the state-reward transition distribution $P_1\left(s',r|s,w,m\right)$, the confounder transition distribution $P_2\left(w|s\right)$, the intermediate state transition distribution $P_3\left(m|s,a\right)$, and the policy $\pi\left(a|s\right)$. Note, that the policy $\pi\left(a|s\right)$ does not depend on $w$ since we assume that $w$ is unobserved to the agent.

\section{Confounded Reinforcement Learning}
\label{section:alg_for_unobserved}

\begin{figure}[htbp]
    \centering
	\includegraphics[width=239.39438pt]{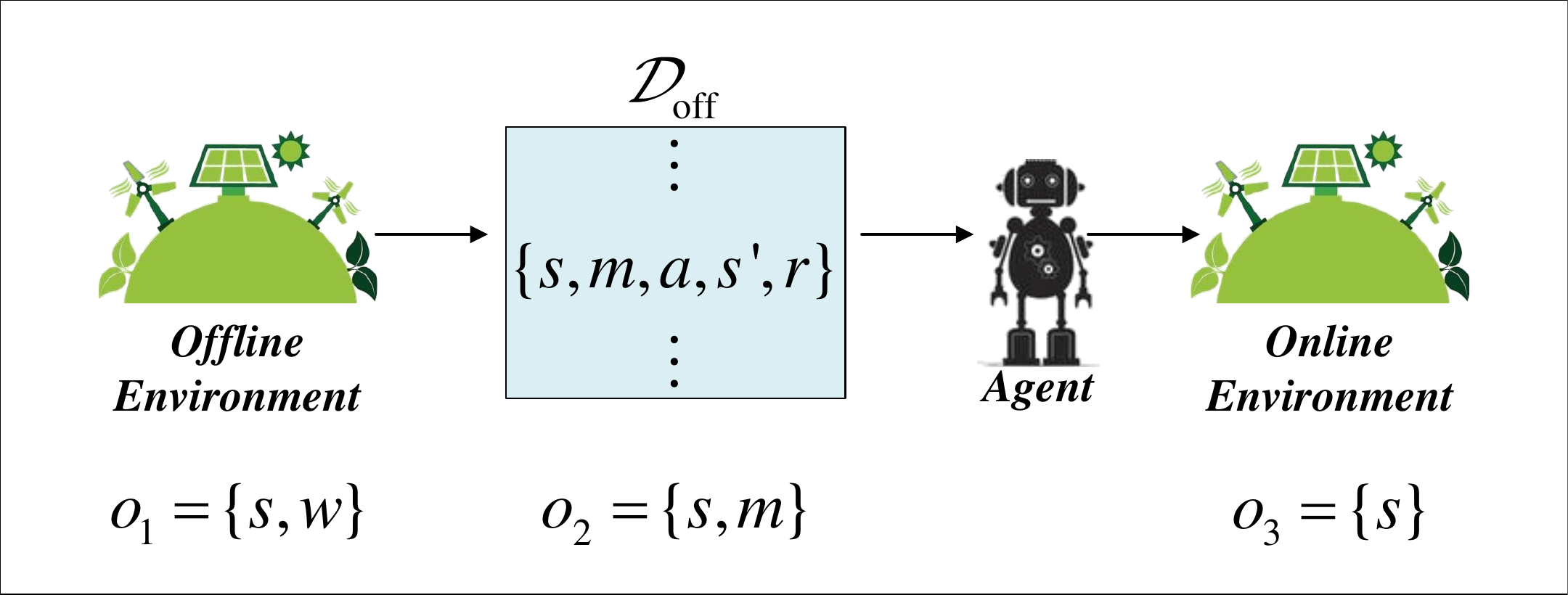}
    \caption{Diagram of the data flow framework, where the confounders in the offline data (i.e., $\mathcal{D}_{\offText}$) are unobserved.}
    \label{figure:dataflow_unobserved}
\end{figure}

Figure~\ref{figure:dataflow_unobserved} describes RL tasks where offline RL algorithms need to use confounded offline data to train an agent that will be tested in the online environment.
Specifically, we assume that the observation $o_1$ of a person/agent in the offline environment includes the state set $s$ and the confounder set $w$. This person interacts with the environment several times with some fixed strategy. 
Note that the data generation process in the offline environment conforms to the SCM in the offline setting.
At the same time, there are other sensors that keep observing the process of interaction of this person/agent with the environment and gather it into an offline dataset. We assume that the observation $o_2$ of these sensors includes the state set $s$ and the intermediate state set $m$. 
The offline RL algorithms need to utilize the offline dataset to train an agent that will be tested in the online environment, where the observation $o_3$ includes only the state set $s$. 
Note that the online data generation process in the testing environment conforms to the SCM in the online setting.
In summary, the confounders are unobserved in both the offline and the online data.

The above task setting leads to different dynamics $\hat{\PText}(s',r|s,a)$ and $\bar{\PText}(s',r|s,a)$ in offline data and online data (as proven in Appendix~A), where $\hat{\PText}$ and $\bar{\PText}$ denote the probability distributions corresponding to the SCMs in the offline and online settings respectively. Because the dynamics in the data determine the optimal value functions, the optimal value functions corresponding to offline data and online data are different.
Therefore, if an agent is trained directly by the original deep RL algorithms using the offline data, the estimated optimal value function of this agent is optimal w.r.t. the offline data and suboptimal w.r.t. the online data. Thus, this agent will perform poorly in the online environment.
We can also understand the above problem from another perspective, i.e., by analyzing the loss function of the original deep RL algorithms that satisfy Assumption~\ref{assumption:drlloss}.
\begin{assumption}
    \label{assumption:drlloss}
    The loss function of the neural network (NN) of the deep RL algorithm only depends on the current state $s$, the action $a$, the next state $s'$, and the reward $r$.
\end{assumption}

The loss function of the original deep RL algorithms which satisfy Assumption~\ref{assumption:drlloss} is shown as follows:
\begin{equation}
    \label{eqn:loss1_cmdp2}
    \begin{split}
    L_1\left(\phi,\mathcal{D}_{\offText}\right)
    &\triangleq\mathbb{E}_{s,a,s',r\sim{\mathcal{D}_{\offText}}}\left[f_\phi+h_\phi\right]\\
    &=\mathbb{E}_{s,a\sim\mathcal{D}_{\offText}}\left[\mathbb{E}_{s', r\sim \hat{\PText}\left(\cdot,\cdot|s,a\right)}\left[
   f_\phi+h_\phi\right]\right],\\
    \end{split}
\end{equation}
where $\mathcal{D}_{\offText}$ denotes all offline data collected through observation, $\phi$ denotes the NN parameters to be optimized, $f_\phi+h_\phi$ denotes the ``loss" for a single training sample which differs between different NNs, $f_\phi$ denotes the part of the ``loss" that depends on $(s,a,s',r)$, and $h_\phi$ denotes the part of the ``loss" that depends only on $(s,a)$. Table~\ref{table:fh_func} gives some examples of $f_{\phi}$ and $h_{\phi}$ for different NNs of the deep RL algorithms. 
It is obvious that the loss function in Equation~\ref{eqn:loss1_cmdp2} corresponds to the dynamics in the offline data, and thus the agent trained using this loss function performs poorly in the online environment.

To address this problem, based on importance sampling and do-calculus~\cite{pearl2009causality,pearl2012calculus}, we propose two deconfounding methods, namely, reweighting method and resampling method, which can both be combined with existing deep RL algorithms such as soft actor-critic (SAC)~\cite{sac}, deep Q-learning (DQN)~\cite{dqn}, double deep Q-learning (DDQN)~\cite{ddqn111} and conservative Q-learning (CQL)~\cite{cql}, provided that a weak condition (i.e., Assumption~\ref{assumption:drlloss}) is satisfied. 
The deconfounding RL algorithms, i.e., the deep RL algorithms combined with the deconfounding methods, can estimate the optimal value function corresponding to the online data, and accordingly the agent trained by these deconfounding RL algorithms is expected to perform well in the online environment.
In this section, we assume that the confounders are unobserved in the offline data. We also provide analysis of  partially observed confounders in the offline data, and derive the reweighting and resampling methods in Appendix~B accordingly.

\begin{table}[bt]
\vskip 0.15in
\begin{center}
\begin{small}
\begin{sc}
\begin{tabular}{lcccr}
\toprule
NN & $f_\phi$ & $h_\phi$ \\
\midrule
SAC Actor    & -&$-V\left(s\right)$\\
DQN Q & $\left(y_1\left(s',r\right) - Q_{\phi}\left(s, a\right)\right)^2$&-\\
DDQN Q    & $\left(y_2\left(s',r\right) - Q_{\phi}\left(s, a\right)\right)^2$&-\\
SAC Critic    & $\left(y_3\left(s',r\right) - Q_{\phi}\left(s, a\right)\right)^2$&-\\
\bottomrule
\end{tabular}
\end{sc}
\end{small}
\end{center}
\vskip -0.1in
\caption{This table describes the different meanings of $f_{\phi}$ and $h_{\phi}$ for different NNs. DQN Q, DDQN Q and SAC Critic denote the Q-networks of DQN, DDQN and SAC, respectively. SAC Actor denotes the actor network of SAC. $y_1\left(s',r\right)=r+\gamma \max_{a'} Q_{\phi'}\left(s', a'\right)$. $y_2\left(s',r\right)=r + \gamma Q_{\phi'}\left(s', \argmaxText_{a'} Q_\phi\left(s',a'\right)\right)$. $y_3\left(s',r\right)=r + \gamma \left(V\left(s'\right)\right)$. $V\left(s\right) = \pi_\phi \left(s\right)^T \left[Q_\phi\left(s\right) - \alpha \log \left(\pi_\phi \left(s\right)\right)\right]$.}
\label{table:fh_func}
\end{table}

\subsection{Reweighting Method}

As mentioned above, since the loss function $L_1\left(\phi,\mathcal{D}_{\offText}\right)$ of the original deep RL algorithms corresponds to the dynamics in the offline data, the agent trained by the original deep RL algorithms performs poorly in the online environment. This inspires us to modify $L_1\left(\phi,\mathcal{D}_{\offText}\right)$ to $L_2\left(\phi,\mathcal{D}_{\offText}\right)$ as follows:
\begin{equation}
    \label{eqn:loss2_cmdp2}    L_2\left(\phi,\mathcal{D}_{\offText}\right)\triangleq\mathbb{E}_{s,a\sim{\mathcal{D}_{\offText}}}\left[\mathbb{E}_{s',r\sim{\bar{\PText}\left(\cdot,\cdot|s,a\right)}}\left[f_\phi+h_\phi\right]\right].
\end{equation}
Clearly, $L_2\left(\phi,\mathcal{D}_{\offText}\right)$ corresponds to the dynamics in the online data. However we cannot directly estimate $L_2\left(\phi,\mathcal{D}_{\offText}\right)$ in the form of Equation~\ref{eqn:loss2_cmdp2} from the offline data.
Therefore, we transform $L_2\left(\phi,\mathcal{D}_{\offText}\right)$ into a form that can be estimated from the offline data as follows:

\begin{proposition}
    \label{proposition:1}
    Under the definitions of the CMDP and SCMs in Section~\ref{background}, it holds that
    \begin{equation}
    \label{eqn:loss2_cmdp2_derive}
    \begin{split}
    &L_2\left(\phi,\mathcal{D}_{\offText}\right)\triangleq\mathbb{E}_{s,a\sim{\mathcal{D}_{\offText}}}\left[\mathbb{E}_{s',r\sim{\bar{\PText}\left(\cdot,\cdot|s,a\right)}}\left[f_\phi+h_\phi\right]\right]\\
    &=\mathbb{E}_{s,a\sim{\mathcal{D}_{\offText}}}\left[\mathbb{E}_{s',r,m\sim{\hat{\PText}\left(\cdot,\cdot,\cdot|s,a\right)}}\left[d_1\left(\tau\right)f_\phi+h_\phi\right]\right]\\
    &=\mathbb{E}_{s,a,s',r,m\sim{\mathcal{D}_{\offText}}}\left[d_1\left(\tau\right)f_\phi+h_\phi\right],\\
    \end{split}
    \end{equation}
    where $\tau$ is a shorthand for the set of variables $(s,m,a,s',r)$ and $d_1\left(\tau\right)$ is defined as follows:
    \begin{equation}
    \label{eqn:label_definition}
    \begin{split}
    d_1\left(\tau\right)=\dfrac{\sum\limits_{a'}\hat{\PText}\left(s',r|m,a',s\right)\hat{\PText}\left(a'|s\right)}{\hat{\PText}\left(s',r|m,a,s\right)}.\\
    \end{split}
    \end{equation}
\end{proposition}
\begin{proof}
See Appendix~A.
\end{proof}

As shown in Equation~\ref{eqn:loss1_cmdp2} and Equation~\ref{eqn:loss2_cmdp2_derive}, the only difference between $L_2\left(\phi,\mathcal{D}_{\offText}\right)$ in its new form and $L_1\left(\phi,\mathcal{D}_{\offText}\right)$ is that $f_\phi$ is multiplied by an extra weight $d_1\left(\tau\right)$. Therefore, we only need to estimate $d_1\left(\tau\right)$ and modify the loss function of the deep RL algorithms from $L_1\left(\phi,\mathcal{D}_{\offText}\right)$ to $L_2\left(\phi,\mathcal{D}_{\offText}\right)$.
Note that $d_1\left(\tau\right)$ is composed of several conditional density functions. 
So, as long as these conditional density functions are estimated from the offline data, we can get an estimate of $d_1\left(\tau\right)$.
To estimate these conditional density functions containing both discrete and continuous random variables,
we adopt the LSCDE technique combined with a trick called \textit{adding jitter} or \textit{jittering}~\cite{nagler2018generic}  to add noises to all discrete random variables. It has been theoretically justified that this trick works well if all the noises are chosen from a specific class of noise distribution~\cite{nagler2018generic}. The settings of choosing noises are given in Appendix~C.
 
\subsection{Resampling Method}

The essence of the reweighting method is to adjust the impact of different samples on the loss function by reweighting the offline data. In other words, by incorporating weights into the loss function, we expand the impact of the samples that are less likely to occur in the offline data than in the online data, i.e., $\hat{\PText}(s',r|s,a)<\bar{\PText}(s',r|s,a)$, and vice versa.
However, we can also adjust the impact of different samples on the loss function through adjusting the probability that each sample occurs in the offline data. This idea inspires us to define a loss function as follows:
\begin{equation}
    \label{eqn:loss_resample_unobserved}
    \begin{split}
    &L_3\left(\phi,\mathcal{D}_{\offText}\right)\triangleq\mathbb{E}_{I{\sim}p_1}\left[f_{\phi,I}+h_{\phi,I}\right],\\
    \end{split}
\end{equation}
where $p_1\left(I=i\right)=d_{1,i}/\sum_{j=1}^{N}d_{1,j}$, $d_{1,i}$ is a shorthand for $d_1\left(s_{i},m_{i},a_{i},s_{i}',r_{i}\right)$ where $s_{i},a_{i},r_{i},s_{i}',m_{i}$ denotes the state, action, reward, next state and intermediate state in the offline dataset ${D}_{off}$, respectively, $f_{\phi,I}$ is a shorthand for $f_\phi\left(s_{I},a_{I},s_{I}',r_{I}\right)$ and $h_{\phi,I}$ for $h_\phi\left(s_{I},a_{I}\right)$.
In $L_3\left(\phi,\mathcal{D}_{\offText}\right)$, we modify the probability that each sample occurs in the offline data according to the value of $d_1\left(\tau\right)$.
By the definition of $L_3\left(\phi,\mathcal{D}_{\offText}\right)$, based on the importance sampling technique and reparameterization trick, we can prove that the deep RL algorithms combined with the resampling method can also learn the optimal policy w.r.t. the online data from the offline data as follows:
\begin{proposition}
    \label{proposition:l2eql3__cmdp2}
    Under the definitions of the CMDP and SCMs in Section~\ref{background}, the loss function of the resampling method is asymptotically equal to that of the reweighting method as in Equation~\ref{eqn:loss_resample_equal_cmdp__2} provided that the dataset is large enough.
    \begin{equation}
    \label{eqn:loss_resample_equal_cmdp__2}
    \begin{split}
    &\lim_{N\to\infty}\left(L_3\left(\phi,\mathcal{D}_{\offText}\right)-L_2\left(\phi,\mathcal{D}_{\offText}\right)\right)=0\\
    \end{split}
\end{equation}
\end{proposition}

\begin{proof}
See Appendix~A.
\end{proof}
According to Proposition~\ref{proposition:l2eql3__cmdp2}, the resampling method is asymptotically equivalent to the reweighting method and can also deconfound the offline data provided that the offline dataset is large enough.

The advantages of the above two deconfounding methods are obvious. On the one hand, since both the reweighting and resampling methods decouple the deconfounding process from the RL algorithm, these methods can be easily combined with existing offline RL algorithms provided that the weak condition in Assumption~1 is satisfied. On the other hand, the implementation of these methods is quite straightforward, only requiring minor modification to the original RL algorithms, i.e., 
we only need to incorporate the estimated $d_1\left(\tau\right)$ into the loss function.

\subsection{Simplification of the Causal Models}
The convergence rate of conditional
density estimation (CDE) decreases exponentially as the dimension of the variables increases due to the curse of dimensionality. Therefore, it is necessary to simplify the causal models by reducing the dimensions of the variables in the process of CDE, so that our deconfounding methods can be applied to problems with higher-dimensional variables. For example, if the confounder exists only between the action and reward, we can simplify the causal model correspondingly by ignoring the variable $s'$ in the process of CDE. We formally define two types of simplified causal models for two specific cases and derive the corresponding simplified forms of $d_{1}\left(\tau\right)$ in Appendix~D.

\section{Results}
\label{results}

Evaluation of the deconfounding method is a challenging issue due to the lack of benchmark tasks. In addition, there is little work in deep RL on learning from observational data with confounders. 
To this end, we design four benchmark tasks, namely, EmotionalPendulum, WindyPendulum, EmotionalPendulum*, and WindyPendulum*, by modifying the Pendulum task in the OpenAI Gym~\cite{gym}.
All the implementations of the offline RL algorithms in this paper follow d3rlpy, an offline RL library~\cite{seno2021d3rlpy}.
All the hyperparameters of the offline RL algorithms are set to the default values of d3rlpy. The rewards are tested over 20 episodes every 1000 learning steps, and averaged over 5 random seeds.
Other hyperparameters and the implementation details are described in Appendix~C. 

\subsection{Unobserved Confounders}

\begin{figure}[htbp]
    \centering
    \subfigure{
        \includegraphics[width=1.55in]{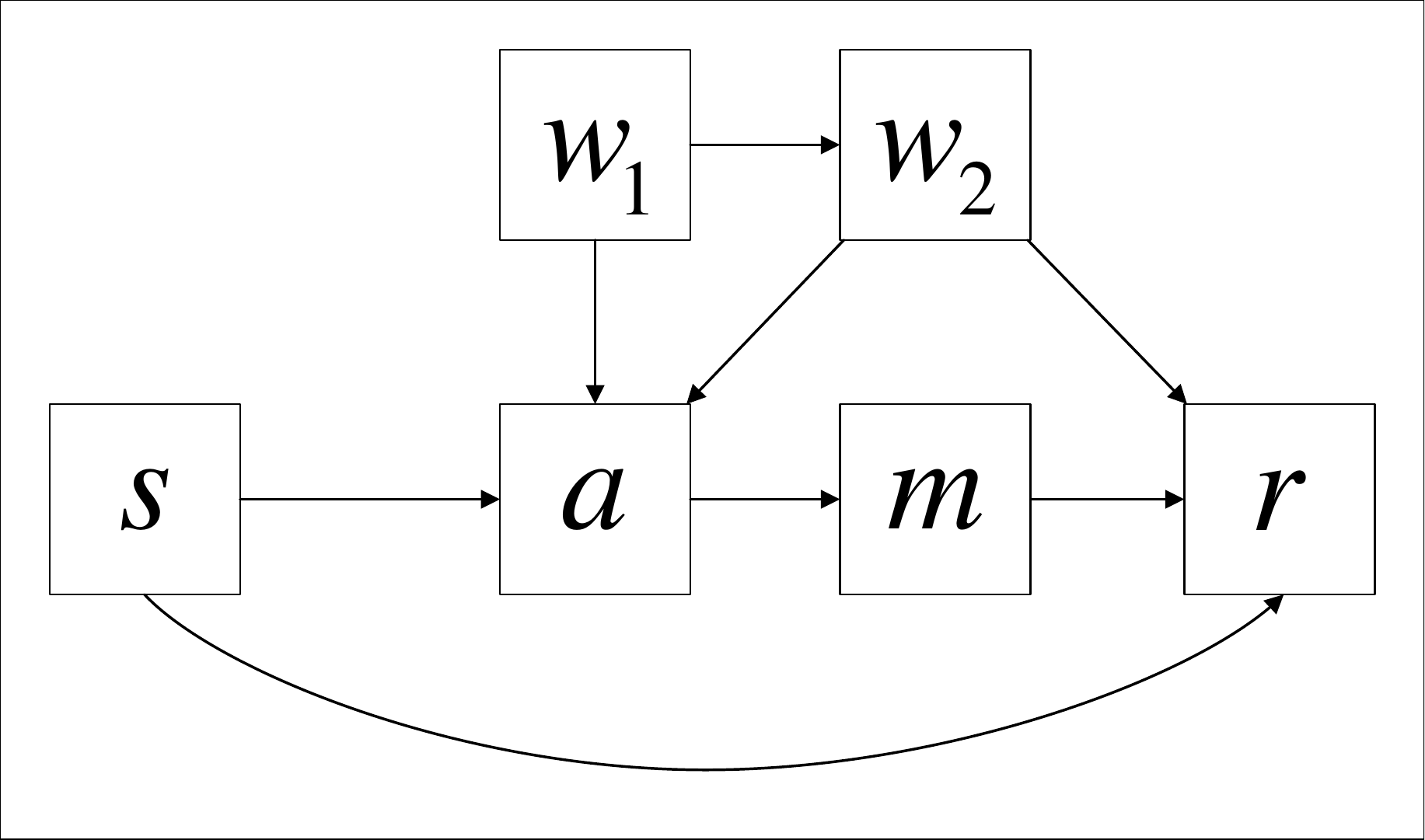}
        \label{658930}
    }
    \subfigure{
	\includegraphics[width=1.55in]{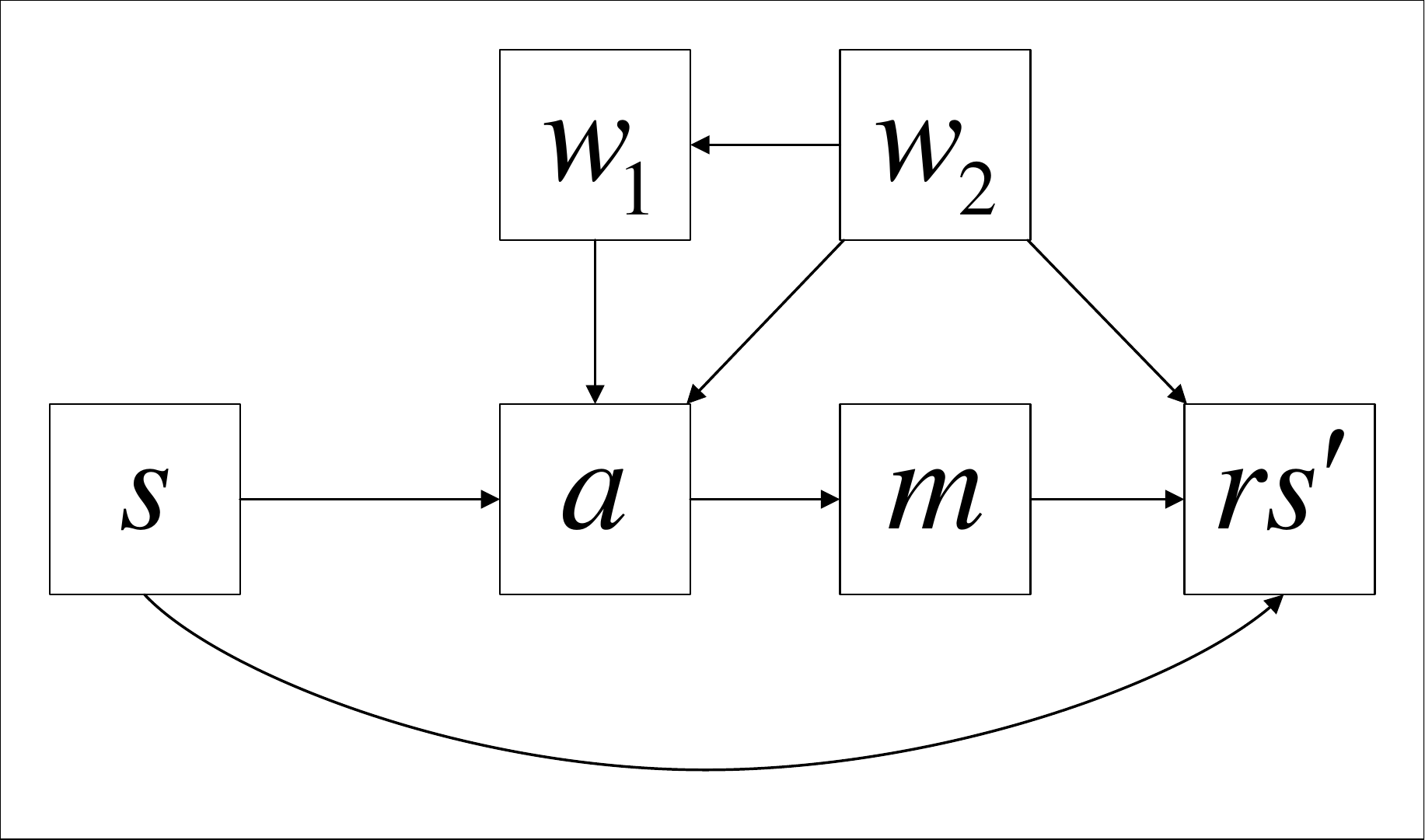}
	\label{1234703}
    }
    \\
    \subfigure{
        \includegraphics[width=1.55in]{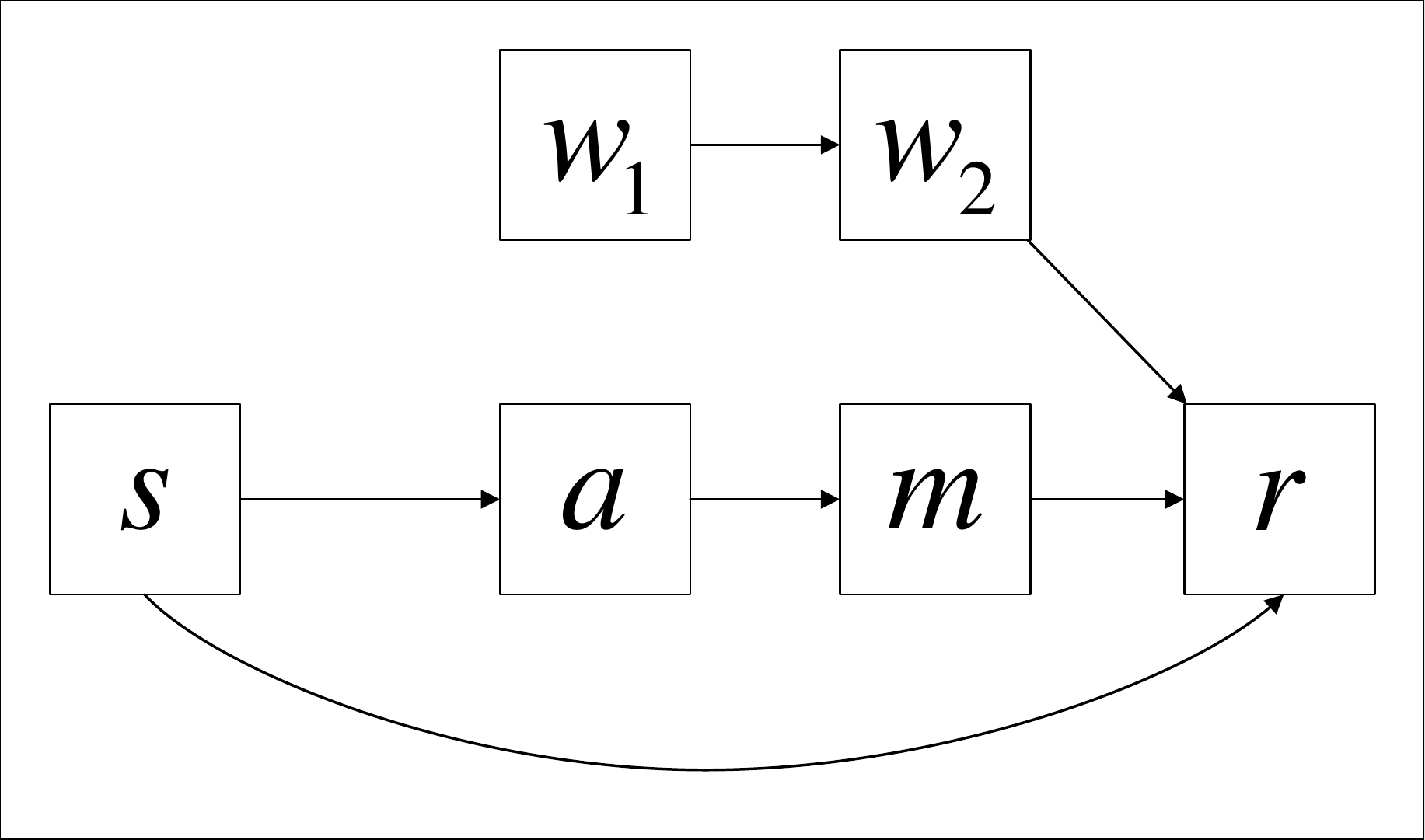}
    }
    \subfigure{
	\includegraphics[width=1.55in]{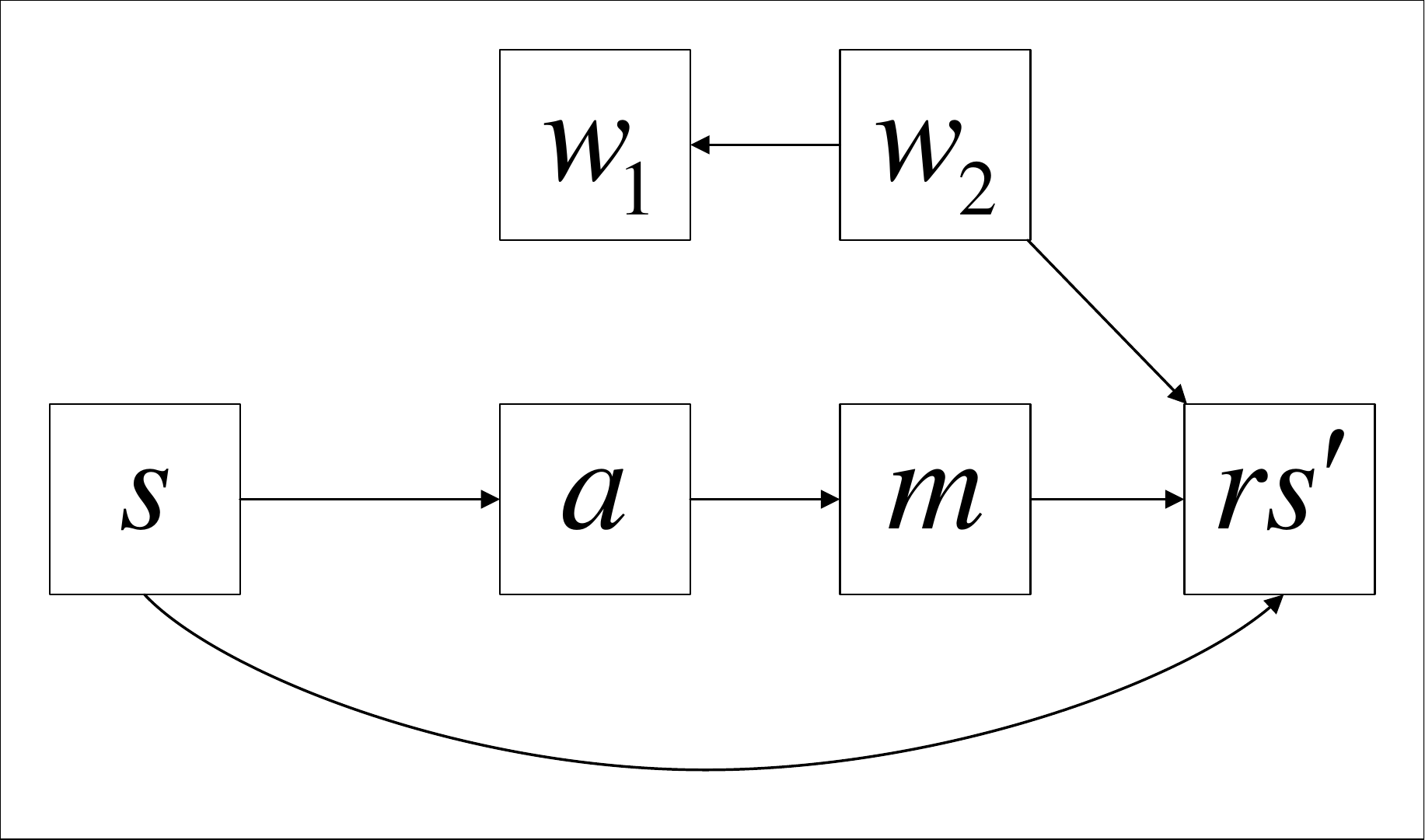}
    }
    \caption{\textbf{Left}: the causal graphs depicting the offline and online data generating processes in EmotionalPendulum. \textbf{Right}: the causal graphs depicting the offline and online data generating processes in WindyPendulum. In both tasks, the confounders in the offline data are unobserved.}
    \label{figure:scm_unobserved}
\end{figure}

Figure~\ref{figure:scm_unobserved} shows the graphical models for EmotionalPendulum and WindyPendulum, where the confounders are unobserved in the offline data.
Obviously, these graphical models are special cases of the graphical models described in Section~\ref{background}.
In both tasks, we assume that there is an entertainment facility similar to a pendulum. One end of the pendulum is attached to a fixed point and a human sitting in a seat at the other end of the pendulum needs to swing the pendulum to an upright position. 
However, the controller of the pendulum fails intermittently, i.e., the action $a\in\left\{-2,-1,0,1,2\right\}$ that the human wants to take may not equal the actual action $m\in\left\{-2,-1,0,1,2\right\}$ executed by the machine. 
Specifically, in most cases, $m$ is equal to $a$, and there is a probability $p_{fail}$ that $m$ randomly selects an action in the action space, with each action equally likely to be selected.
There are some sensors are used to collect the confounded offline data generated in the above human-environment interaction process. Note that the environmental information $o_2$ in these offline data includes the state $s$ and the actual action $m$ executed by the machine.
The offline RL algorithms need to use these confounded offline data to train an agent that substitutes the human to control the pendulum in the testing environment, where the confounders are unobserved, i.e., $o_3=\left\{s\right\}$.
The main difference between the two tasks is that the confounders in EmotionalPendulum are between the action and reward while the confounders in WindyPendulum are between the action, reward and next state.

The design of EmotionalPendulum refers to the real-world phenomenon that humans may take some actions that are not rational when they are in negative emotions such as boredom and fear. 
Specifically, we assume that the human sitting in the entertainment facility may feel afraid and decide to slow down if the speed is too fast (i.e., if the speed $|v|$ is above the threshold $v_T$), or feel boring and decide to speed up if the speed is too slow (i.e., if $|v| \leq v_T$), and will make rational decisions according to the trained agent if the emotions of the human are not negative, i.e., the behavior policy of the human depends on $o_1=\left\{s,w_1,w_2\right\}$, where $w_1$ denotes whether the human has negative emotions and $w_2$ denotes whether the expressions of the human are negative.
Then, the environment will return a reward $r=r_o+r_a$, where $r_o$ denotes the original reward in the Pendulum task and $r_a$ denotes an additional reward that is generated by the environment to encourage the human if the emotions are negative. 

In WindyPendulum, we take into consideration the wind that may change the direction arbitrarily at each step. We assume that the wind force is 2.5 times larger than the largest force applied to the pendulum by the human/agent and the human will feel afraid if there is a strong wind. The human in a state of fear may choose the force opposite to the wind or decide to slow down, or choose the rational action, i.e., the behavior policy of the human depends on $o_1=\left\{s,w_1,w_2\right\}$, where $w_1$ denotes whether the human is afraid due to the wind and $w_2$ denotes the direction of the wind. See Appendix~E for more details of the description of the two tasks.
Similar to EmotionalPendulum, a reward $r=r_o+r_a$ is then received, where the additional reward $r_a$ is used to encourage the human if the wind exists.

\begin{figure*}[ht]
  \centering
  \includegraphics[width=\textwidth]{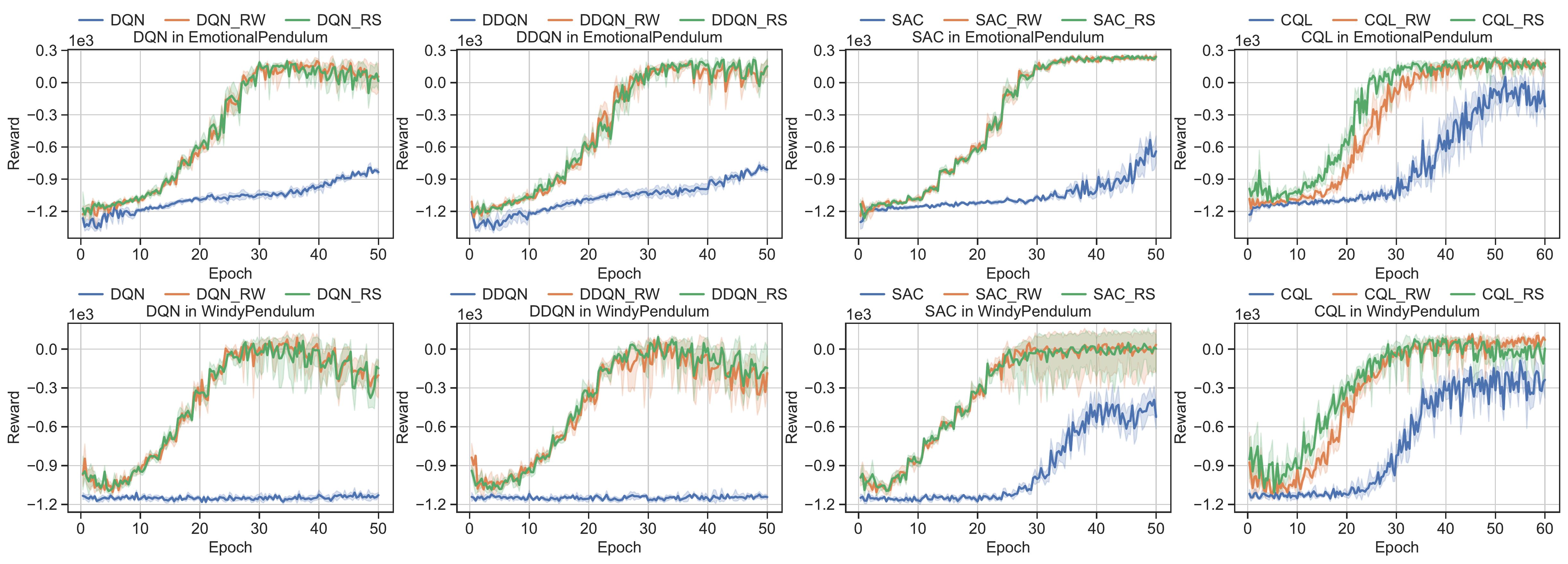}
  \caption{\footnotesize Performance of the deep RL algorithms with and without our deconfounding methods in EmotionalPendulum and WindyPendulum.}
  \label{fig:unobserved_plot}
\end{figure*}

\begin{figure*}[ht]
  \centering
  \includegraphics[width=\textwidth]{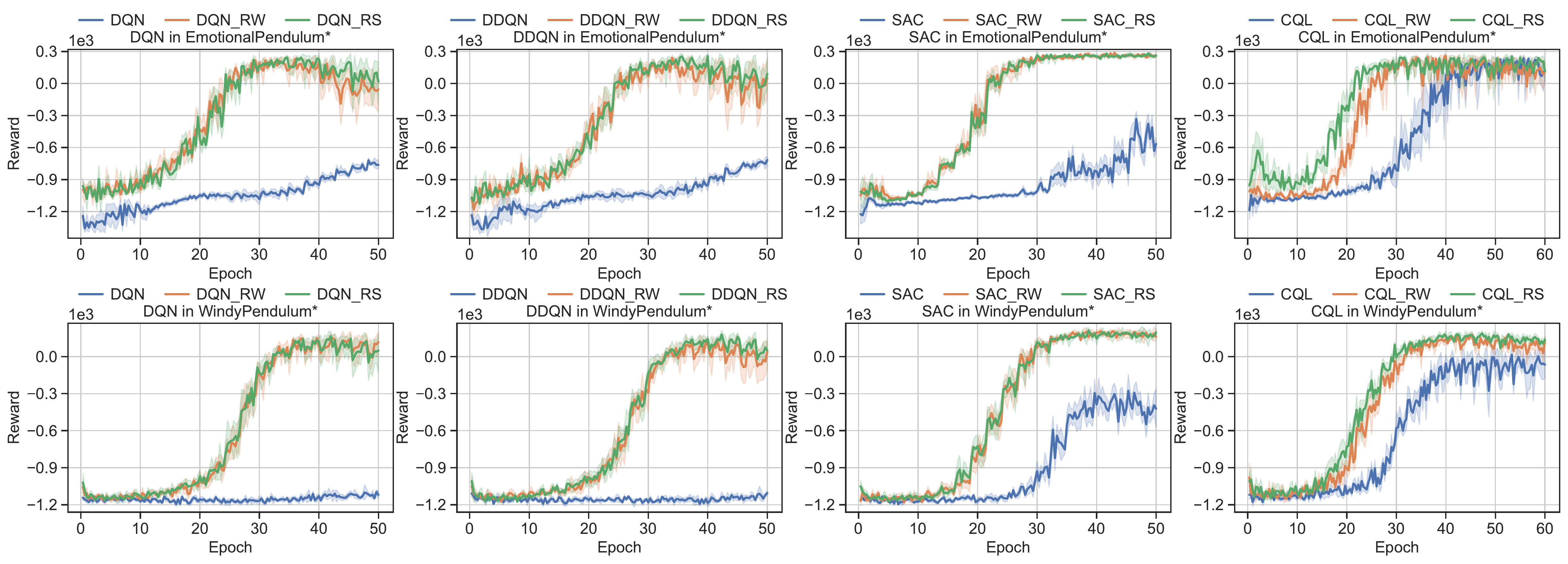}
  \caption{\footnotesize Performance of the deep RL algorithms with and without our deconfounding methods in EmotionalPendulum* and WindyPendulum*.}
  \label{fig:partially_observed_plot}
\end{figure*}

As shown in Figure~\ref{fig:unobserved_plot}, we compare the deep RL algorithms DQN, DDQN, SAC and CQL with and without our deconfounding methods in EmotionalPendulum and WindyPendulum, respectively. 
DQN\_RW denotes the DQN algorithm that uses the reweighting method and DQN\_RS for the resampling method. Other algorithms are denoted similarly. It is clear that the deconfounding deep RL algorithms perform better than the original deep RL algorithms in both tasks. 

We take EmotionalPendulum as an example to specifically illustrate the confounding problem and why our proposed methods could perform better than the original deep RL algorithms.
As shown in the top left of Figure~\ref{figure:scm_unobserved}, the association flowing along the directed path $a\rightarrow m \rightarrow r$ is causal association, while the association flowing along $a \leftarrow w1 \rightarrow w2 \rightarrow r$ is non-causal association. 
In other words, because the human may take irrational actions when he/she is in negative emotions, and because the environment may return some additional reward if the emotions of the human are negative, there is non-causal association between the irrational action and the additional reward, which may mislead the human to believe that the irrational action will be well rewarded.
Given $s$, there is only causal association between $a$ and $r$ in the online data, because the non-causal association flowing along $a\leftarrow s\rightarrow r$ is blocked by $s$. In contrast, given $s$, there are both causal association and non-causal association between $a$ and $r$ in the offline data.
Therefore, the original deep RL algorithms trained using confounded offline data will be misled by non-causal association in the data, and thus perform poorly in the online environment. However, the deconfounding deep RL algorithms can learn causal association from confounded offline data, and thus perform well in the online environment.

\begin{table}[t]
\vskip 0.15in
\begin{center}
\begin{small}
\begin{sc}

\resizebox{0.97\linewidth}{!}{

\begin{tabular}{llll|r|r|r|r}
\toprule
    &   &     &     &      BC &    CQL &         CQL\_RW &         CQL\_RS \\
$p_{fail}$ & $odds_1$ & $v_T$ & $I_{p,1}$ &         &        &                 &                 \\
\midrule
\multirow{8}{*}{0.2} & \multirow{4}{*}{4} & \multirow{2}{*}{0.5} & 0.7 &  -1020.0 & -236.3 &           114.4 &  \textbf{153.9} \\
    &   &     & 0.9 & -1145.9 & -552.5 &            59.1 &  \textbf{100.6} \\
\cline{3-8}
    &   & \multirow{2}{*}{1.0} & 0.7 &  -857.6 &  -98.5 &           180.1 &  \textbf{181.8} \\
    &   &     & 0.9 & -1096.5 & -531.5 &            85.2 &  \textbf{113.6} \\
\cline{2-8}
\cline{3-8}
    & \multirow{4}{*}{6} & \multirow{2}{*}{0.5} & 0.7 &  -361.2 &   46.3 &   \textbf{92.4} &            84.6 \\
    &   &     & 0.9 &  -999.2 &  -70.7 &            36.8 &   \textbf{72.3} \\
\cline{3-8}
    &   & \multirow{2}{*}{1.0} & 0.7 &  -399.9 &   77.9 &            93.3 &   \textbf{95.1} \\
    &   &     & 0.9 &  -834.5 &  -25.6 &   \textbf{83.3} &            76.9 \\
\cline{1-8}
\cline{2-8}
\cline{3-8}
\multirow{8}{*}{0.1} & \multirow{4}{*}{4} & \multirow{2}{*}{0.5} & 0.7 & -977.2 &  -82.0 &           115.0 &  \textbf{123.1} \\
    &   &     & 0.9 & -1126.7 & -523.1 &   \textbf{97.2} &            44.5 \\
\cline{3-8}
    &   & \multirow{2}{*}{1.0} & 0.7 &  -884.7 &   81.7 &           206.5 &  \textbf{225.2} \\
    &   &     & 0.9 & -1065.8 & -443.5 &            56.3 &  \textbf{144.3} \\
\cline{2-8}
\cline{3-8}
    & \multirow{4}{*}{6} & \multirow{2}{*}{0.5} & 0.7 &  -557.1 &   94.2 &           130.3 &  \textbf{131.3} \\
    &   &     & 0.9 & -1022.0 &  -10.0 &            81.6 &  \textbf{115.3} \\
\cline{3-8}
    &   & \multirow{2}{*}{1.0} & 0.7 &  -570.9 &  119.2 &  \textbf{126.5} &           117.1 \\
    &   &     & 0.9 & -899.5 &   80.4 &           110.3 &  \textbf{112.1} \\
\bottomrule
\end{tabular}

}

\end{sc}
\end{small}
\end{center}
\vskip -0.1in
\caption{Comparison of BC, CQL, CQL\_RW, and CQL\_RS under different settings of EmotionalPendulum.}
\label{table:cql_and_bc_EmotionalPendulum_uc}
\end{table}

So, what will happen if we enhance the non-causal association flowing along $a \leftarrow w1 \rightarrow w2 \rightarrow r$ in the offline data?
The values of the environmental hyperparameters of EmotionalPendulum, which correspond to the top row of Figure~\ref{fig:unobserved_plot}, are as follows: $p_{fail}=0.2$, $odds_1=4$, $v_T=1$, and $I_{p,1}=0.7$, where $I_{p,1}$ denotes the probability of the human choosing irrational actions when he/she has negative emotions and $odds_1$ denotes the odds that the human does not have negative emotions.
Clearly, if we keep the other hyperparameters constant and increase the value of $I_{p,1}$ from 0.7 to 0.9, the non-causal association flowing along $a \leftarrow w1 \rightarrow w2 \rightarrow r$ in the offline data will be enhanced.
As shown in Table~\ref{table:cql_and_bc_EmotionalPendulum_uc}, the enhancement of the non-causal association leads to increases in the gaps between the performance of CQL and CQL\_RW, and between the performance of CQL and CQL\_RS. 
Additionally, Table~\ref{table:cql_and_bc_EmotionalPendulum_uc} shows that, no matter how much $p_{fail}$, $odds_1$, and $v_T$ are, the gaps increase if we increase the value of $I_{p,1}$ from 0.7 to 0.9.
Similarly, if we keep the other hyperparameters constant, reducing the value of $odds_1$ leads to an enhancement of the non-causal association. As shown in Table~\ref{table:cql_and_bc_EmotionalPendulum_uc}, regardless of how much $p_{fail}$, $v_T$, and $I_{p,1}$ are, the gaps increase if we reduce the value of $odds_1$ from 6 to 4.

Moreover, CQL\_RW and CQL\_RS perform better than CQL under all settings in Table~\ref{table:cql_and_bc_EmotionalPendulum_uc}, which verifies the robustness of our deconfounding methods. In addition, as shown in Table~\ref{table:cql_and_bc_EmotionalPendulum_uc}, the behavior cloning (BC) algorithm performs much worse than CQL, CQL\_RW and CQL\_RS because the offline data are not expert data. 
In both EmotionalPendulum and WindyPendulum, the behavior policy of the human performs poorly, because the human takes many irrational actions based on his/her emotions. 
Therefore, imitation learning methods such as BC are not suitable for both tasks.
In fact, in addition to Table~\ref{table:cql_and_bc_EmotionalPendulum_uc}, we compare BC, the original deep RL algorithms, and the deconfounding deep RL algorithms under different settings in Appendix~E to verify the robustness of our deconfounding methods.
Furthermore, we perform some ablation experiments in Appendix~E.

\subsection{Partially Observed Confounders}

In the above two tasks, we assume that the confounders in the offline data are unobserved and 
identify the causal effect based on a variation of the frontdoor criterion.
By contrast, in the two new tasks, EmotionalPendulum* and WindyPendulum*,
we assume that the confounders in the offline data are partially observed and condition on a subset of the confounders to identify the causal effect.
In the new tasks, the controller of the pendulum never fails, i.e., the action $a$ that the human wants to take is the same as the actual action $m$ executed by the machine.
The corresponding causal graphs and details of these two new tasks are given in Appendix~E. 
To solve the new tasks, we define a new CMDP and two new SCMs and derive two new deconfounding methods under the new definitions in Appendix~B.
As shown in Figure~\ref{fig:partially_observed_plot}, the deconfounding deep RL algorithms perform better than the original deep RL algorithms. Additionally, in Appendix~E, we compare BC, the original deep RL algorithms, and the deconfounding deep RL algorithms under different settings.

\section{Related Work}
\label{section:related_work}

A number of studies~\cite{lattimore2016causal,nair2021spectral} seek to introduce causality into the field of RL in various settings, where the use of observational data~\cite{wang2020provably,lu2018deconfounding,gasse2021causal} is one of the key issues.

\paragraph{The Bandit Problem.}
The Causal bandit problem~\cite{lattimore2016causal} seeks to learn the optimal intervention from the interventional data, which conform to a causal graph, to minimise a simple regret.
Sen {\it et al.}~\shortcite{sen2017identifying} propose the successive rejects algorithms and derive the gap-dependent bounds for causal bandit based on importance sampling.
The Contextual bandit problem~\cite{NIPS2007_4b04a686} seeks to learn a policy which depends on the context, namely the environmental variables, to maximize the reward which depends on the action and the context.
Kallus and Zhou~\shortcite{kallus2018policy} propose algorithms for OPE and learning from the confounded observational data with continuous actions based on the inverse probability weighting.
Our paper, however, studies the causal RL problem which is more difficult than the bandit problem with a longer horizon.

\paragraph{Causal RL.}
There is some work studying causal RL in the model-based RL settings.
Lu {\it et al.}~\shortcite{lu2018deconfounding} propose a model-based RL method that estimates an SCM from observational data with the time-invariant confounder between the action and reward based on the latent-variable model proposed by Louizos {\it et al.}~\shortcite{Louizos2017Causal}. This model, however, cannot estimate the correct causal effect provided that the latent variable is misspecified or the data distribution is complex as shown by Rissanen and Marttinen~\shortcite{Rissanen2021critical}.
Gasse {\it et al.}~\shortcite{gasse2021causal} combine interventional data with observational data to estimate the latent-based transition model, which can be used for model-based RL. However, since that the time consumed by their program increases rapidly as the size of the discrete latent space increases,
and that they assume that the state space is smaller than the discrete latent space, their algorithm can only be used to address problems with small state spaces. 
On the contrary, our deconfounding methods can be applied to problems with both continuous and discrete variables and our assumptions do not restrict the size of the state and confounder spaces.
The closest work is the one by Wang {\it et al.}~\shortcite{wang2020provably}, where the authors focus on how to improve the sample efficiency of the online algorithm by incorporating large amounts of observational data in the model-free RL settings.
However, they follow the assumption of Yang and Wang~\shortcite{yang2019sample,yang2020reinforcement}; Jin {\it et al.}~\shortcite{jin2020provably} that the transition kernels and reward functions are linear, so that the corresponding SCMs are also linear~\cite{peters2017elements}. This strong assumption is hard to hold in the real world, where the dynamics of the environment are nonlinear.
Furthermore, they assume that there is very complex correlation between the backdoor-adjusted or frontdoor-adjusted feature and another feature. For example, they assume that the backdoor-adjusted feature is the expectation of the state-action-confounder feature. These assumptions make their theoretical algorithms difficult to implement.
In contrast, the transition kernels and reward functions in our assumptions can be nonlinear, and experiments are provided to verify the performance of the offline RL algorithms combined with our deconfounding methods.

\paragraph{Off-Policy Evaluation.}
A line of work in the OPE field uses confounded data to estimate the performance of the evaluation policy. There are several papers that design estimators for the partially observable Markov decision process in the tabular setting~\cite{nair2021spectral,tennenholtz2020off}, which cannot be applied to the large or continuous state space. Recently a lot of work~\cite{kallus2020confounding,gelada2019off,hallak2017consistent,kallus2019efficiently,liu2018breaking} proposes algorithms for OPE based on the importance sampling technique. However, these algorithms estimate the stationary distribution density ratio, which changes as the evaluation policy changes. This makes their work inapplicable to the RL field, where the evaluation policy is always changing. Instead, since we realize that SAC, DQN, and other off-policy RL algorithms utilize only the single-step transition in the data, our deconfounding methods are based on the conditional distribution density ratio, which remains constant in the learning process, so that our work can be applied to the RL field.

\section{Conclusion}
In this paper, we propose two plug-in deconfounding methods based on the importance sampling and causal inference techniques for model-free deep RL algorithms. These two methods can be applied to large and continuous environments.
In addition, we prove that our deconfounding methods can construct an unbiased loss function w.r.t the online data and show that the deconfounding deep RL algorithms perform better than the original deep RL algorithms in the four benchmark tasks 
which are created by modifying the OpenAI Gym.
A limitation of our work is the assumption that we already know the causal graph of the environment. In future work, we plan to build a causal model through causal discovery without prior knowledge of the environment and to deconfound the confounded data based on the discovered causal model. We hope that our deconfounding methods can bridge the gap between offline RL algorithms and real-world problems by leveraging large amounts of observational data.

\appendix

\section*{Acknowledgments}
This work is supported in part by the National Key Research and Development Program of China (No.~2021ZD0112400), the NSFC-Liaoning Province United Foundation under Grant U1908214, the National Natural Science Foundation of China (No.~62076259), the 111 Project (No.~D23006), the Fundamental Research Funds for the Central Universities under grant DUT21TD107, DUT22ZD214, the LiaoNing Revitalization Talents Program (No.~XLYC2008017), the Fundamental and Applicational Research Funds of Guangdong province (No.~2023A1515012946), and the Fundamental Research Funds for the Central Universities-Sun Yat-sen University.
The authors would like to thank Zifan Wu for revising the author response to the reviewers.



\bibliographystyle{named}
\bibliography{ijcai23}

\onecolumn

\section*{\centering{Causal Deep Reinforcement Learning using Observational Data}}
\section*{\centering{Appendix}}

\section{Mathematical Details}

\subsection{Difference between the Dynamics in the Offline Data and the Online Data}
\label{appendix:ponpoffnotequal}
In Section~\ref{section:alg_for_unobserved}, we claim that the dynamics are different between offline data and online data. Now we prove it formally.
\begin{proof}

From the law of total probability and the criterion for the identification of covariate-specific effects, which is indicated by the Rule 2 in Chapter 3 of~\cite{glymour2016causal}, we obtain:
\begin{equation}
    \label{eqn:poff}
    \begin{split}
        &\hat{\PText}\left(s',r|s,a\right)=\sum\limits_{w}\hat{\PText}\left(s',r|a,w,s\right)\hat{\PText}\left(w|s,a\right)
    \end{split}
\end{equation}
and
\begin{equation}
    \label{eqn:pon}
    \begin{split}
        \bar{\PText}\left(s',r|s,a\right)&=\hat{\PText}\left(s',r|s,do\left(a\right)\right)\\
        &=\sum\limits_{w}\hat{\PText}\left(s',r|a,w,s\right)\hat{\PText}\left(w|s\right).\\
    \end{split}
\end{equation}

Since $a$ depends on $w$ given $s$ in the offline data, $\hat{\PText}\left(w|s,a\right)\not\equiv\hat{\PText}\left(w|s\right)$. Therefore the above two are different.

\end{proof}

\subsection{Proof of Proposition~\ref{proposition:1}}
\label{appendix:unobserved_reweight_loss_proof}

\begin{proposition*}
    Under the definitions of the CMDP and SCMs in Section~\ref{background}, it holds that
    \begin{equation}
    \label{eqn:loss2_cmdp2_derive_inproof}
        \begin{split}
        L_2\left(\phi,\mathcal{D}_{\offText}\right)&\triangleq\mathbb{E}_{s,a\sim{\mathcal{D}_{\offText}}}\left[\mathbb{E}_{s',r\sim{\bar{\PText}\left(\cdot,\cdot|s,a\right)}}\left[f_\phi+h_\phi\right]\right]\\
        &=\mathbb{E}_{s,a\sim{\mathcal{D}_{\offText}}}\left[\mathbb{E}_{s',r,m\sim{\hat{\PText}\left(\cdot,\cdot,\cdot|s,a\right)}}\left[d_1\left(s,m,a,s',r\right)f_\phi+h_\phi\right]\right]\\
        &=\mathbb{E}_{s,a,s',r,m\sim{\mathcal{D}_{\offText}}}\left[d_1\left(s,m,a,s',r\right)f_\phi+h_\phi\right]\\
        \end{split}
    \end{equation}
                            \end{proposition*}

\begin{proof}
From the importance sampling, we obtain:
\begin{equation*}
    \begin{split}
    L_2\left(\phi,\mathcal{D}_{off}\right)&\triangleq\mathbb{E}_{s,a\sim{\mathcal{D}_{off}}}\left[\mathbb{E}_{s',r\sim{\bar{\PText}\left(\cdot,\cdot|s,a\right)}}\left[f_\phi+h_\phi\right]\right]\\
    &=\mathbb{E}_{s,a\sim{\mathcal{D}_{off}}}\left[\mathbb{E}_{s',r,m\sim{\hat{\PText}\left(\cdot,\cdot,\cdot|s,a\right)}}\left[\frac{\bar{\PText}\left(s',r,m|s,a\right)}{\hat{\PText}\left(s',r,m|s,a\right)}f_\phi\right]+h_\phi\right]\\
    &=\mathbb{E}_{s,a,s',r,m\sim{\mathcal{D}_{off}}}\left[\frac{\bar{\PText}\left(s',r,m|s,a\right)}{\hat{\PText}\left(s',r,m|s,a\right)}f_\phi+h_\phi\right].\\
    \end{split}
\end{equation*}

According to do-calculus, we obtain:
\begin{equation*}
    \label{eqn:ratio1_cmdp2}
    \begin{split}
    \frac{\bar{\PText}\left(s',r,m|s,a\right)}{\hat{\PText}\left(s',r,m|s,a\right)}&=\frac{\hat{\PText}\left(s',r,m|s,do\left(a\right)\right)}{\hat{\PText}\left(s',r,m|s,a\right)}\\
    &=\frac{\hat{\PText}\left(s',r|m,s,do\left(a\right)\right)\hat{\PText}\left(m|s,do\left(a\right)\right)}{\hat{\PText}\left(s',r|m,s,a\right)\hat{\PText}\left(m|s,a\right)}.\\
    \end{split}
\end{equation*}

According to do-calculus, we obtain: $\hat{\PText}\left(m|s,do\left(a\right)\right)=\hat{\PText}\left(m|s,a\right)$, and
\begin{equation*}
    \begin{split}
    \hat{\PText}\left(s',r|m,s,do\left(a\right)\right)&=\hat{\PText}\left(s',r|s,do\left(m\right),do\left(a\right)\right)\\
    &=\hat{\PText}\left(s',r|s,do\left(m\right)\right)\\
    &=\sum\limits_{a'}\hat{\PText}\left(s',r|m,a',s\right)\hat{\PText}\left(a'|s\right).\\
    \end{split}
\end{equation*}

By bringing in the two equations, we obtain:
\begin{equation*}
    \begin{split}
    &\frac{\bar{\PText}\left(s',r,m|s,a\right)}{\hat{\PText}\left(s',r,m|s,a\right)}=\frac{\sum\limits_{a'}\hat{\PText}\left(s',r|m,a',s\right)\hat{\PText}\left(a'|s\right)}{\hat{\PText}\left(s',r|m,s,a\right)}.\\
    \end{split}
\end{equation*}

By bringing in this equation, we obtain:
\begin{equation*}
    \begin{split}
    &L_2\left(\phi,\mathcal{D}_{off}\right)=\mathbb{E}_{s,a,s',r,m\sim{\mathcal{D}_{off}}}\left[\frac{\sum\limits_{a'}\hat{\PText}\left(s',r|m,a',s\right)\hat{\PText}\left(a'|s\right)}{\hat{\PText}\left(s',r|m,a,s\right)}f_\phi+h_\phi\right].\\
    \end{split}
\end{equation*}

\end{proof}

\subsection{Proof of Proposition~\ref{proposition:l2eql3__cmdp2}}
\label{appendix:unobserved_resample_loss_proof}

\begin{proposition*}
    Under the definitions of the CMDP and SCMs in Section~\ref{background}, the loss function of the resampling method is asymptotically equal to that of the reweighting method as in Equation~\ref{eqn:loss_resample_equal_cmdp__2_inproof__} provided that the dataset is large enough.
    \begin{equation}
    \label{eqn:loss_resample_equal_cmdp__2_inproof__}
        \begin{split}
        &\lim_{N\to\infty}\left(L_3\left(\phi,\mathcal{D}_{\offText}\right)-L_2\left(\phi,\mathcal{D}_{\offText}\right)\right)=0\\
        \end{split}
    \end{equation}
\end{proposition*}

\begin{proof}
Slightly modifying the first step of the proof in Appendix~\ref{appendix:unobserved_reweight_loss_proof}, we obtain:
\begin{equation*}
    \begin{split}
    L_2\left(\phi,\mathcal{D}_{off}\right)&\triangleq\mathbb{E}_{s,a\sim{\mathcal{D}_{off}}}\left[\mathbb{E}_{s',r\sim{\bar{\PText}\left(\cdot,\cdot|s,a\right)}}\left[f_\phi+h_\phi\right]\right]\\
    &=\mathbb{E}_{s,a\sim{\mathcal{D}_{off}}}\left[\mathbb{E}_{s',r,m\sim{\hat{\PText}\left(\cdot,\cdot,\cdot|s,a\right)}}\left[\frac{\bar{\PText}\left(s',r,m|s,a\right)}{\hat{\PText}\left(s',r,m|s,a\right)}\left(f_\phi+h_\phi\right)\right]\right]\\
    &=\mathbb{E}_{s,a,s',r,m\sim{\mathcal{D}_{off}}}\left[d_1\left(s,m,a,s',r\right)\left(f_\phi+h_\phi\right)\right].\\
    \end{split}
\end{equation*}

Then, based on the reparameterization trick, we obtain:
\begin{equation*}
    \begin{split}
    &L_2\left(\phi,\mathcal{D}_{off}\right)=\mathbb{E}_{I{\sim}DiscreteU\left(1,N\right)}\left[d_{1,I}\left(f_\phi\left(s_{I},a_{I},s_{I}',r_{I}\right)+h_\phi\left(s_{I},a_{I}\right)\right)\right],\\
    \end{split}
\end{equation*}
where $DiscreteU\left(1,N\right)$ denotes the discrete uniform distribution on the integers $1,2,\dotsc,N$.

Furthermore, from the importance sampling, we obtain:
\begin{equation*}
    \begin{split}
    &L_2\left(\phi,\mathcal{D}_{off}\right)=\mathbb{E}_{I{\sim}p_1}\left[\frac{\frac{1}{N}}{p_1\left(I\right)}d_{1,I}\left(f_\phi\left(s_{I},a_{I},s_{I}',r_{I}\right)+h_\phi\left(s_{I},a_{I}\right)\right)\right].\\
    \end{split}
\end{equation*}

And then, the formula can be simplified to
\begin{equation*}
    \begin{split}
    L_2\left(\phi,\mathcal{D}_{off}\right)=&\mathbb{E}_{I{\sim}p_1}\left[\frac{\sum\limits_{j=1}^{N}d_{1,j}}{N}\left(f_\phi\left(s_{I},a_{I},s_{I}',r_{I}\right)+h_\phi\left(s_{I},a_{I}\right)\right)\right]\\
    =&L_3\left(\phi,\mathcal{D}_{off}\right)\frac{\sum\limits_{j=1}^{N}d_{1,j}}{N}.\\
    \end{split}
\end{equation*}

Thus,
\begin{equation*}
    \begin{split}
    \lim_{N\to\infty}\left(L_3\left(\phi,\mathcal{D}_{\offText}\right)-L_2\left(\phi,\mathcal{D}_{\offText}\right)\right)&=\lim_{N\to\infty}L_3\left(\phi,\mathcal{D}_{off}\right)\left(1-\frac{\sum\limits_{j=1}^{N}d_{1,j}}{N}\right).\\
    \end{split}
\end{equation*}

Since
\begin{equation*}
    \begin{split}
    \lim_{N\to\infty}\left(1-\frac{\sum\limits_{j=1}^{N}d_{1,j}}{N}\right)=&1-\lim_{N\to\infty}\mathbb{E}_{I{\sim}DiscreteU\left(1,N\right)}\left[d_{1,I}\right]\\
    =&1-\lim_{N\to\infty}\mathbb{E}_{s,a,s',r,m\sim{\mathcal{D}_{off}}}\left[d_1\left(s,m,a,s',r\right)\right]\\
    =&1-\lim_{N\to\infty}\mathbb{E}_{s,a\sim{\mathcal{D}_{off}}}\left[\mathbb{E}_{s',r,m\sim{\hat{\PText}\left(\cdot,\cdot,\cdot|s,a\right)}}\left[\frac{\bar{\PText}\left(s',r,m|s,a\right)}{\hat{\PText}\left(s',r,m|s,a\right)}\right]\right]\\
    =&1-\lim_{N\to\infty}\mathbb{E}_{s,a\sim{\mathcal{D}_{off}}}\left[\mathbb{E}_{s',r,m\sim{\bar{\PText}\left(\cdot,\cdot,\cdot|s,a\right)}}\left[1\right]\right]\\
    =&0,\\
    \end{split}
\end{equation*}
and the limit of the loss function of the neural network is finite, we obtain:
\begin{equation*}
    \begin{split}
    \lim_{N\to\infty}\left(L_3\left(\phi,\mathcal{D}_{\offText}\right)-L_2\left(\phi,\mathcal{D}_{\offText}\right)\right)&=\lim_{N\to\infty}L_3\left(\phi,\mathcal{D}_{off}\right)\lim_{N\to\infty}\left(1-\frac{\sum\limits_{j=1}^{N}d_{1,j}}{N}\right)\\
    &=0.\\
    \end{split}
\end{equation*}

\end{proof}

\subsection{Proof of Proposition~\ref{proposition:cmdp1_loss2_holds}}
\label{appendix:partially observed_reweight_loss_proof}

\begin{proposition*}
    Under Assumption~\ref{assumption1} and the definitions of the CMDP and SCMs in Appendix~\ref{appendix:background_for_partially_observed_confounder}, it holds that
    \begin{equation}
        \label{eqn:loss2_cmdp1_inproof}
        \begin{split}
        L_2\left(\phi,\mathcal{D}_{\offText}\right)&\triangleq\mathbb{E}_{s,a\sim{\mathcal{D}_{\offText}}}\left[\mathbb{E}_{s',r\sim{\bar{\PText}\left(\cdot,\cdot|s,a\right)}}\left[f_\phi+h_\phi\right]\right]\\
        &=\mathbb{E}_{s,a\sim{\mathcal{D}_{\offText}}}\left[\mathbb{E}_{s',r,u\sim{\hat{\PText}\left(\cdot,\cdot,\cdot|s,a\right)}}\left[\frac{\hat{\PText}\left(u|s\right)}{\hat{\PText}\left(u|s,a\right)}f_\phi+h_\phi\right]\right]\\
        &=\mathbb{E}_{s,a,s',r,u\sim{\mathcal{D}_{\offText}}}\left[\frac{\hat{\PText}\left(u|s\right)}{\hat{\PText}\left(u|s,a\right)}f_\phi+h_\phi\right].\\
        \end{split}
    \end{equation}
\end{proposition*}

\begin{proof}

From the importance sampling, we obtain:
\begin{equation*}
    \begin{split}
    L_2\left(\phi,\mathcal{D}_{off}\right)&\triangleq\mathbb{E}_{s,a\sim{\mathcal{D}_{off}}}\left[\mathbb{E}_{s',r\sim{\bar{\PText}\left(\cdot,\cdot|s,a\right)}}\left[f_\phi+h_\phi\right]\right]\\
    &=\mathbb{E}_{s,a\sim{\mathcal{D}_{off}}}\left[\mathbb{E}_{s',r,u\sim{\hat{\PText}\left(\cdot,\cdot,\cdot|s,a\right)}}\left[\frac{\bar{\PText}\left(s',r,u|s,a\right)}{\hat{\PText}\left(s',r,u|s,a\right)}f_\phi\right]+h_\phi\right]\\
    &=\mathbb{E}_{s,a,s',r,u\sim{\mathcal{D}_{off}}}\left[\frac{\bar{\PText}\left(s',r,u|s,a\right)}{\hat{\PText}\left(s',r,u|s,a\right)}f_\phi+h_\phi\right].\\
    \end{split}
\end{equation*}

Furthermore, based on do-calculus, we obtain:
\begin{equation*}
    \begin{split}
    L_2\left(\phi,\mathcal{D}_{off}\right)&=\mathbb{E}_{s,a,s',r,u\sim{\mathcal{D}_{off}}}\left[\frac{\hat{\PText}\left(s',r,u|s,do\left(a\right)\right)}{\hat{\PText}\left(s',r,u|s,a\right)}f_\phi+h_\phi\right]\\
    &=\mathbb{E}_{s,a,s',r,u\sim{\mathcal{D}_{off}}}\left[\frac{\hat{\PText}\left(s',r|u,s,do\left(a\right)\right)\hat{\PText}\left(u|s,do\left(a\right)\right)}{\hat{\PText}\left(s',r|u,s,a\right)\hat{\PText}\left(u|s,a\right)}f_\phi+h_\phi\right].\\
    \end{split}
\end{equation*}

According to do-calculus, we obtain:
\begin{equation*}
    \begin{split}
    \hat{\PText}\left(s',r|u,s,do\left(a\right)\right)=\hat{\PText}\left(s',r|u,s,a\right),
    \end{split}
\end{equation*}
and
\begin{equation*}
    \begin{split}
    \hat{\PText}\left(u|s,do\left(a\right)\right)=\hat{\PText}\left(u|s\right).
    \end{split}
\end{equation*}
By bringing in the two equations, we obtain:
\begin{equation*}
    \begin{split}
    &L_2\left(\phi,\mathcal{D}_{off}\right)=\mathbb{E}_{s,a,s',r,u\sim{\mathcal{D}_{off}}}\left[\frac{\hat{\PText}\left(u|s\right)}{\hat{\PText}\left(u|s,a\right)}f_\phi+h_\phi\right].\\
    \end{split}
\end{equation*}

\end{proof}

\subsection{Proof of Proposition~\ref{proposition:l2eql3_cmdp1}}
\label{appendix:partially observed_resample_loss_proof}

\begin{proposition*}
    Under Assumption~\ref{assumption1} and the definitions of the CMDP and SCMs in Appendix~\ref{appendix:background_for_partially_observed_confounder}, the loss function of the resampling method is asymptotically equal to that of the reweighting method as in Equation~\ref{eqn:loss_resample_equal_cmdp__2_inproof} provided that the dataset is large enough.
    \begin{equation}
        \label{eqn:loss_resample_equal_cmdp__2_inproof}
        \begin{split}
        &\lim_{N\to\infty}\left(L_3\left(\phi,\mathcal{D}_{\offText}\right)-L_2\left(\phi,\mathcal{D}_{\offText}\right)\right)=0\\
        \end{split}
    \end{equation}
\end{proposition*}

\begin{proof}

Slightly modifying the first step of the proof in Appendix~\ref{appendix:partially observed_reweight_loss_proof}, we obtain:
\begin{equation*}
    \begin{split}
    L_2\left(\phi,\mathcal{D}_{off}\right)&\triangleq\mathbb{E}_{s,a\sim{\mathcal{D}_{off}}}\left[\mathbb{E}_{s',r\sim{\bar{\PText}\left(\cdot,\cdot|s,a\right)}}\left[f_\phi+h_\phi\right]\right]\\
    &=\mathbb{E}_{s,a\sim{\mathcal{D}_{off}}}\left[\mathbb{E}_{s',r,u\sim{\hat{\PText}\left(\cdot,\cdot,\cdot|s,a\right)}}\left[\frac{\bar{\PText}\left(s',r,u|s,a\right)}{\hat{\PText}\left(s',r,u|s,a\right)}\left(f_\phi+h_\phi\right)\right]\right]\\
    &=\mathbb{E}_{s,a,r,s',u\sim{\mathcal{D}_{off}}}\left[\frac{\hat{\PText}\left(u|s\right)}{\hat{\PText}\left(u|s,a\right)}\left(f_\phi+h_\phi\right)\right].\\
    \end{split}
\end{equation*}

Then, based on the reparameterization trick, we obtain:
\begin{equation*}
    \begin{split}
    &L_2\left(\phi,\mathcal{D}_{off}\right)=\mathbb{E}_{I{\sim}DiscreteU\left(1,N\right)}\left[\frac{\hat{\PText}\left(u_{I}|s_{I}\right)}{\hat{\PText}\left(u_{I}|s_{I},a_{I}\right)}\left(f_\phi\left(s_{I},a_{I},s_{I}',r_{I}\right)+h_\phi\left(s_{I},a_{I}\right)\right)\right].\\
    \end{split}
\end{equation*}

Furthermore, from the importance sampling, we obtain:
\begin{equation*}
    \begin{split}
    &L_2\left(\phi,\mathcal{D}_{off}\right)=\mathbb{E}_{I{\sim}p_2}\left[\frac{\frac{1}{N}}{p_2\left(I\right)}\frac{\hat{\PText}\left(u_{I}|s_{I}\right)}{\hat{\PText}\left(u_{I}|s_{I},a_{I}\right)}\left(f_\phi\left(s_{I},a_{I},s_{I}',r_{I}\right)+h_\phi\left(s_{I},a_{I}\right)\right)\right].\\
    \end{split}
\end{equation*}
And then, the formula can be simplified to
\begin{equation*}
    \begin{split}
    L_2\left(\phi,\mathcal{D}_{off}\right)=&\mathbb{E}_{I{\sim}p_2}\left[\frac{\sum\limits_{j=1}^{N}\frac{\hat{\PText}\left(u_{j}|s_{j}\right)}{\hat{\PText}\left(u_{j}|s_{j},a_{j}\right)}}{N}\left(f_\phi\left(s_{I},a_{I},s_{I}',r_{I}\right)+h_\phi\left(s_{I},a_{I}\right)\right)\right]\\
    =&L_3\left(\phi,\mathcal{D}_{off}\right)\frac{\sum\limits_{j=1}^{N}\frac{\hat{\PText}\left(u_{j}|s_{j}\right)}{\hat{\PText}\left(u_{j}|s_{j},a_{j}\right)}}{N}.\\
    \end{split}
\end{equation*}

Thus,
\begin{equation*}
    \begin{split}
    \lim_{N\to\infty}\left(L_3\left(\phi,\mathcal{D}_{\offText}\right)-L_2\left(\phi,\mathcal{D}_{\offText}\right)\right)&=\lim_{N\to\infty}L_3\left(\phi,\mathcal{D}_{off}\right)\left(1-\frac{\sum\limits_{j=1}^{N}\frac{\hat{\PText}\left(u_{j}|s_{j}\right)}{\hat{\PText}\left(u_{j}|s_{j},a_{j}\right)}}{N}\right).\\
    \end{split}
\end{equation*}

Since
\begin{equation*}
    \begin{split}
    \lim_{N\to\infty}\left(1-\frac{\sum\limits_{j=1}^{N}\frac{\hat{\PText}\left(u_{j}|s_{j}\right)}{\hat{\PText}\left(u_{j}|s_{j},a_{j}\right)}}{N}\right)=&1-\lim_{N\to\infty}\mathbb{E}_{I{\sim}DiscreteU\left(1,N\right)}\left[\frac{\hat{\PText}\left(u_{I}|s_{I}\right)}{\hat{\PText}\left(u_{I}|s_{I},a_{I}\right)}\right]\\
    =&1-\lim_{N\to\infty}\mathbb{E}_{s,a,s',r,u\sim{\mathcal{D}_{off}}}\left[\frac{\hat{\PText}\left(u|s\right)}{\hat{\PText}\left(u|s,a\right)}\right]\\
    =&1-\lim_{N\to\infty}\mathbb{E}_{s,a\sim{\mathcal{D}_{off}}}\left[\mathbb{E}_{s',r,u\sim{\hat{\PText}\left(\cdot,\cdot,\cdot|s,a\right)}}\left[\frac{\bar{\PText}\left(s',r,u|s,a\right)}{\hat{\PText}\left(s',r,u|s,a\right)}\right]\right]\\
    =&1-\lim_{N\to\infty}\mathbb{E}_{s,a\sim{\mathcal{D}_{off}}}\left[\mathbb{E}_{s',r,u\sim{\bar{\PText}\left(\cdot,\cdot,\cdot|s,a\right)}}\left[1\right]\right]\\
    =&0,\\
    \end{split}
\end{equation*}
and the limit of the loss function of the neural network is finite, we obtain:
\begin{equation*}
    \begin{split}
    \lim_{N\to\infty}\left(L_3\left(\phi,\mathcal{D}_{\offText}\right)-L_2\left(\phi,\mathcal{D}_{\offText}\right)\right)&=\lim_{N\to\infty}L_3\left(\phi,\mathcal{D}_{off}\right)\lim_{N\to\infty}\left(1-\frac{\sum\limits_{j=1}^{N}\frac{\hat{\PText}\left(u_{j}|s_{j}\right)}{\hat{\PText}\left(u_{j}|s_{j},a_{j}\right)}}{N}\right)\\
    &=0.\\
    \end{split}
\end{equation*}
\end{proof}

\section{Background and Algorithms for Partially Observed Confounders}

In Section~\ref{section:alg_for_unobserved}, two deconfounding methods are proposed for the RL tasks where the confounders in the offline data are unobserved.
In contrast, in this section, we derive two new deconfounding methods for the new RL tasks where the confounders in the offline data are partially observed.
Before proposing the deconfounding methods, a CMDP and two SCMs
different from those in Section~\ref{background}
are defined to describe  the new RL tasks.

\subsection{Background for Partially Observed Confounders}
\label{appendix:background_for_partially_observed_confounder}

The CMDP for partially observed confounders can be denoted by a seven-tuple $\left\langle \mathcal{S},\mathcal{A},\mathcal{W},\mathcal{R},P_1,P_2,\mu_0\right\rangle$, where $\mathcal{S}$ denotes the state space, $\mathcal{A}$ denotes the discrete action space, $\mathcal{W}$ denotes the confounder space, $\mathcal{R}$ denotes the reward space, $P_1\left(s',r|s,w,a\right)$ denotes the dynamics of this CMDP, $P_2\left(w|s\right)$ denotes the confounder transition distribution, and $\mu_0\left(s\right)$ denotes the initial state distribution.

\begin{figure}[htbp]
    \centering
    \subfigure{
        \includegraphics[width=1.55in]{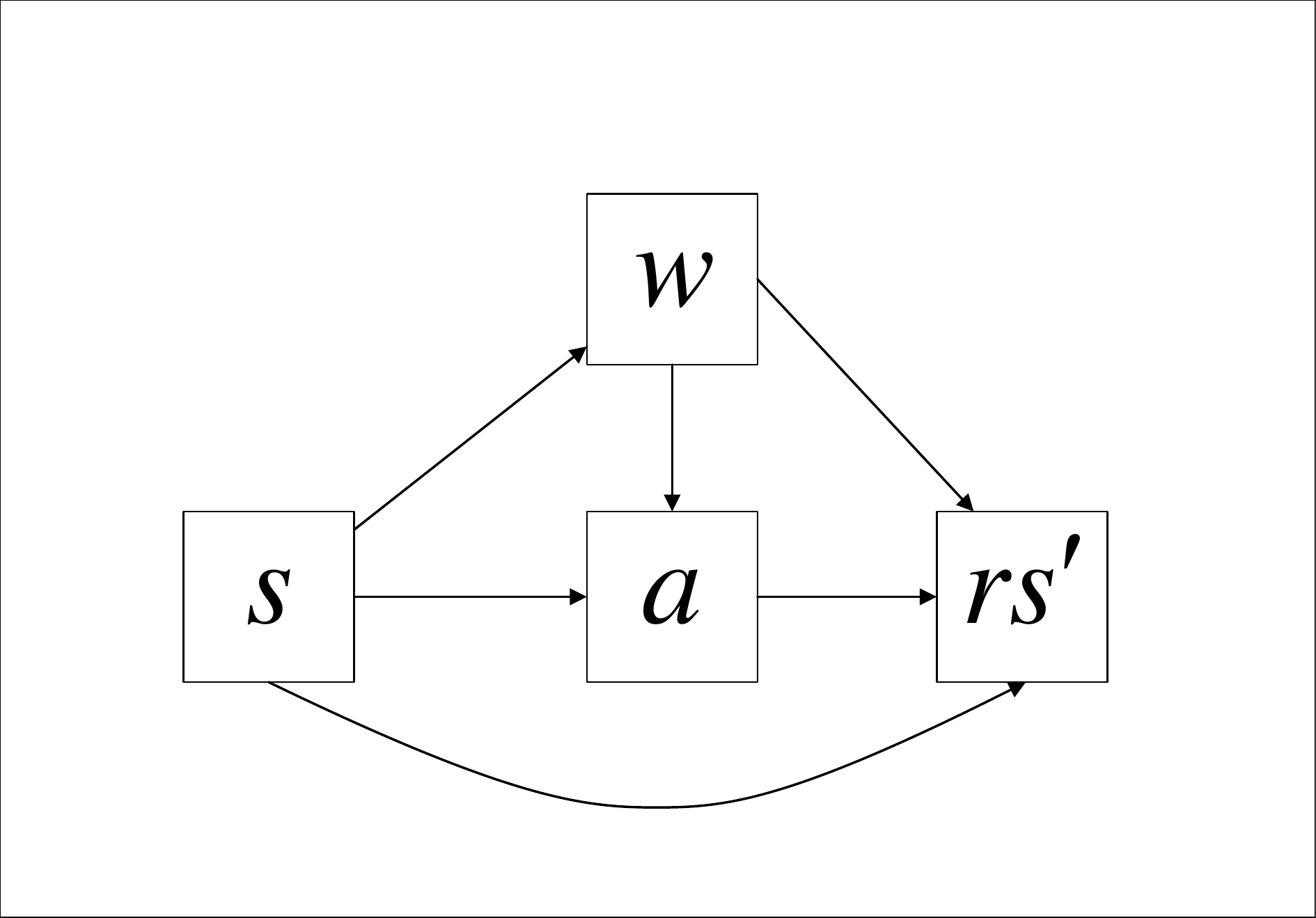}
        \label{cmdp1_scm1}
    }
    \subfigure{
	\includegraphics[width=1.55in]{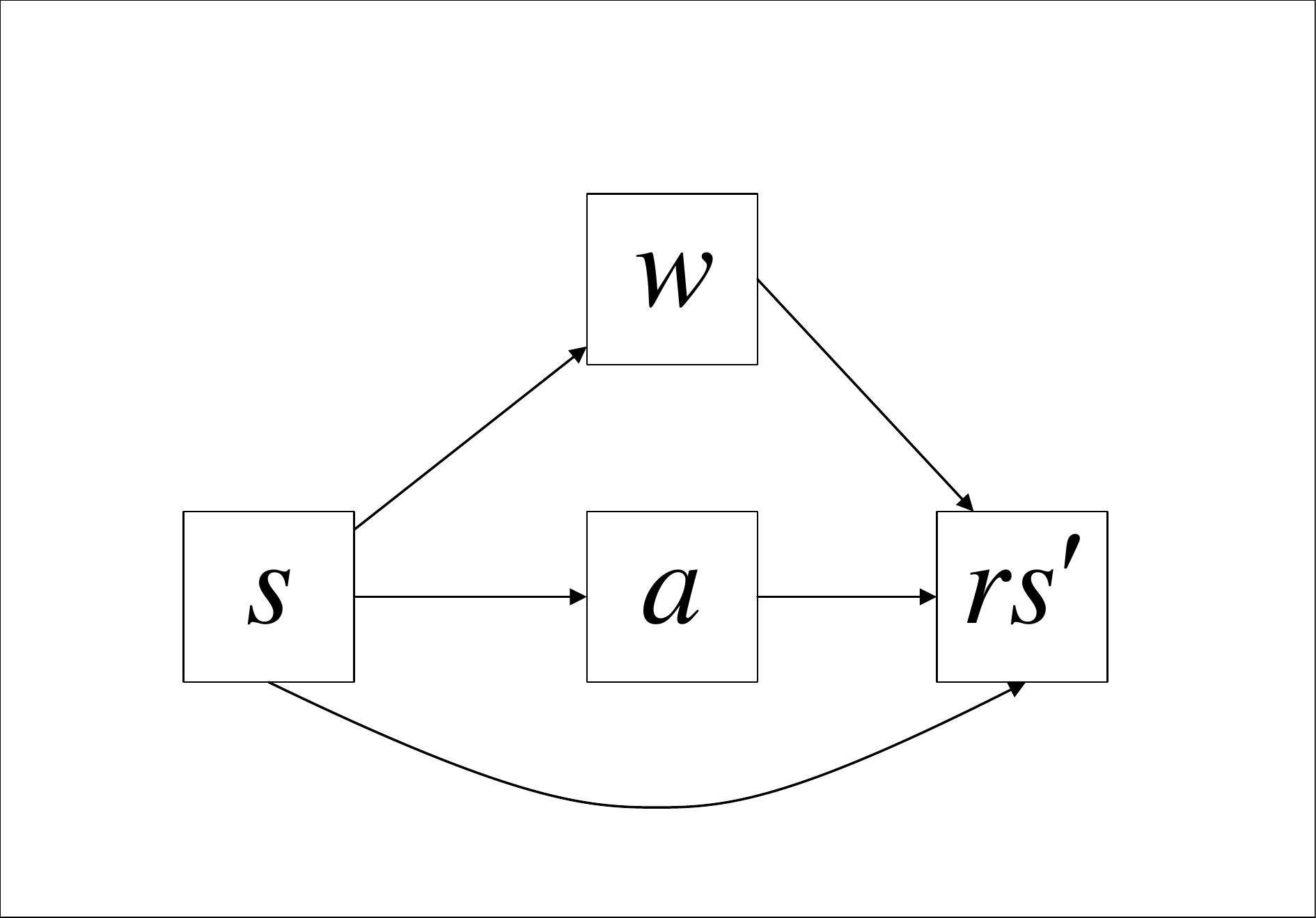}
        \label{cmdp1-scm2}
    }
    \caption{
    The left and right subfigures
    represent the SCM in the offline setting and online setting,
    respectively,
    which correspond to the CMDP for partially observed confounders. The behavior policy depends on $w$ in the offline setting, but not in the online setting.}
    \label{figure:scms}
\end{figure}

The SCM in the offline setting as shown in the first column of Figure~\ref{figure:scms} can be defined as a four-tuple $\left\langle U, V, F, P_e\right\rangle$, where the endogenous variables $V$ includes $\left(s, a, w, s', r\right)$, and the set of structural functions $F$ includes the state-reward transition distribution $P_1\left(s',r|s,w,a\right)$, the confounder transition distribution $P_2\left(w|s\right)$, and the behavior policy $\pi_b\left(a|s,w\right)$. The positivity assumption here is that, for $a \in \mathcal{A}, s \in \mathcal{S},w \in \mathcal{W}$ such that $\PText(s,w)>0$, $\PText(a|s,w)>0$.

The SCM in the online setting where the agent can intervene on the variable $a$ as shown in the second column of Figure~\ref{figure:scms} can be defined as another four-tuple $\left(U, V, F, P_e\right)$, where the set of structural functions $F$ includes the state-reward transition distribution $P_1\left(s',r|s,w,a\right)$, the confounder transition distribution $P_2\left(w|s\right)$, and the policy $\pi\left(a|s\right)$.

We assume that $w$ is partially observable in the offline data, and condition on a subset $u$ of $w$ to identify the causal effect.
Specifically, we assume that there exists an observable subset $u$ of $w$ such that $u{\cup}s$ satisfies the backdoor criterion, as in Assumption~\ref{assumption1}. Two examples are illustrated in Figure~\ref{figure:uofw}. 
\begin{assumption}
    \label{assumption1}
    Suppose there exists a set of observable variables $u{\subseteq}w$, such that $u{\cup}s$ satisfies: $u{\cup}s$ contains no descendant node of $a$ and $u{\cup}s$ blocks every path between $a$ and $s'{\cup}r$ that contains an arrow into $a$.
\end{assumption}

\begin{figure}[ht]
    \centering
    \subfigure{
        \includegraphics[width=1.76in]{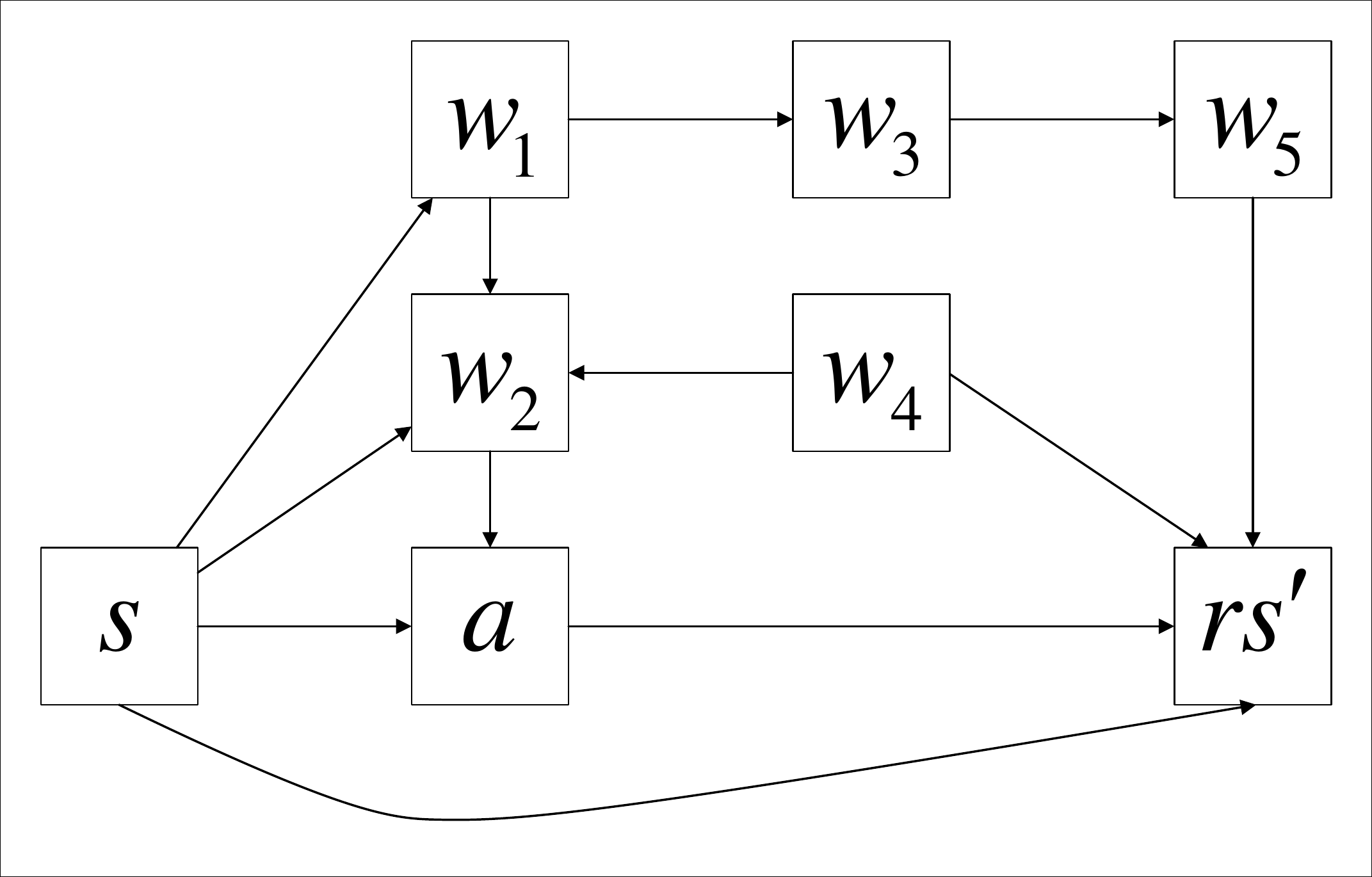}
        
    }
    \subfigure{
	\includegraphics[width=1.5in]{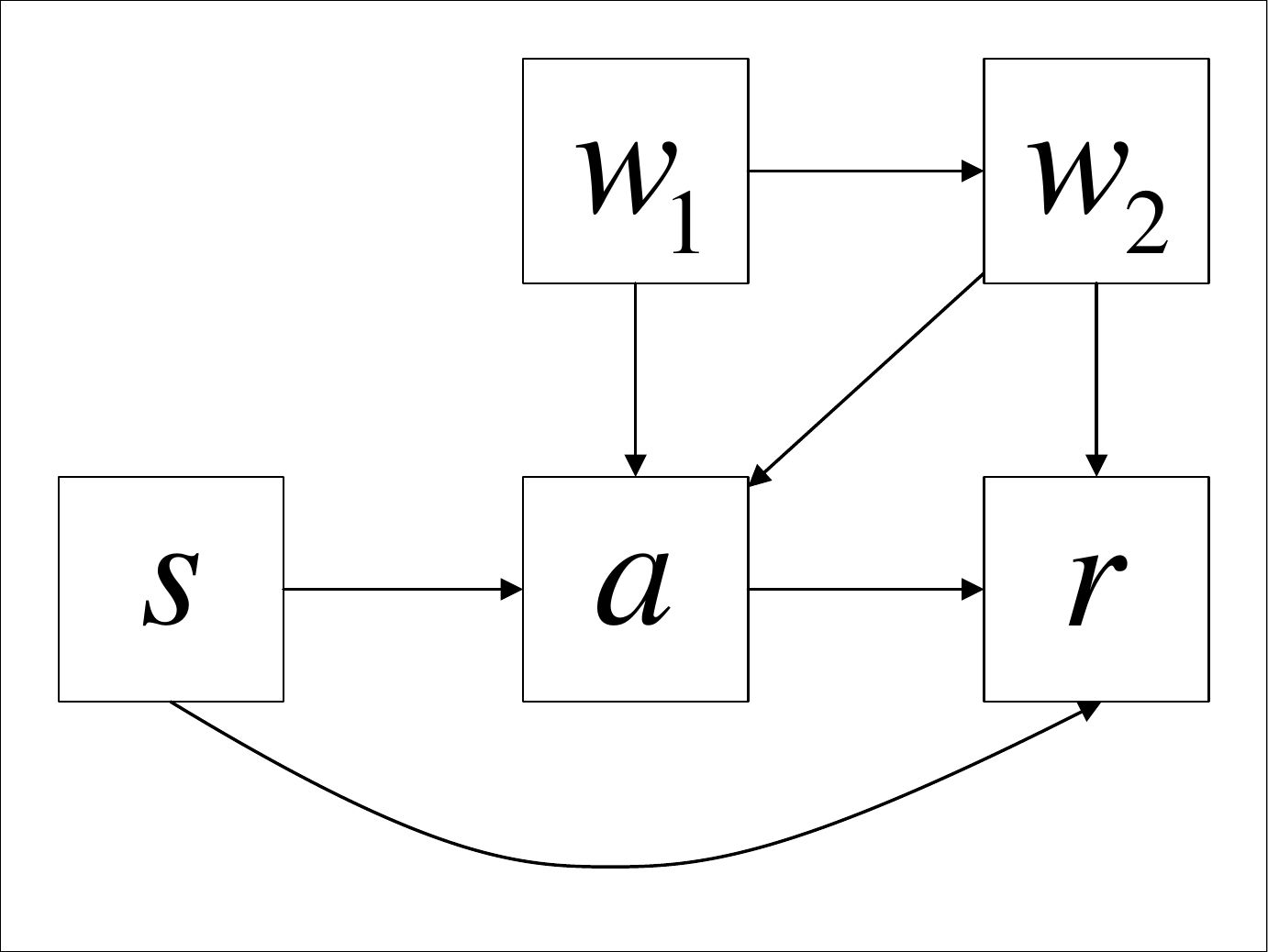}
        
    }
    \vskip -0.05in
    \caption{\textbf{Left}: $u=\left\{w_{1},w_{4}\right\}$ satisfies the Assumption~\ref{assumption1}, which means that $u{\cup}s$ satisfies the backdoor criterion. \textbf{Right}: $u=\left\{w_{2}\right\}$ satisfies the Assumption~\ref{assumption1}, which means that $u{\cup}s$ satisfies the backdoor criterion.}
    \label{figure:uofw}
\end{figure}

\subsection{Algorithms for Partially Observed Confounders}
\label{appendix:algorithms_for_partially_observed_confounder}
Under the definitions of the CMDP and SCMs in this section, two new deconfounding methods are proposed for the new  RL tasks where the confounders in the offline data are partially observed.
As shown in Figure~\ref{figure:dataflow_partially}, the new RL tasks are similar to the RL tasks where the confounders are unobserved in the offline data. Therefore, we only describe the differences between these tasks as follows. First, the intermediate state does not exist in the data generating process anymore. Second, the confounders are partially observed in the offline data, i.e., $o_2=\left\{s,u\right\}$ where $u$ satisfies Assumption~\ref{assumption1}.

\begin{figure}[htbp]
    \centering
    \includegraphics[width=239.39438pt]{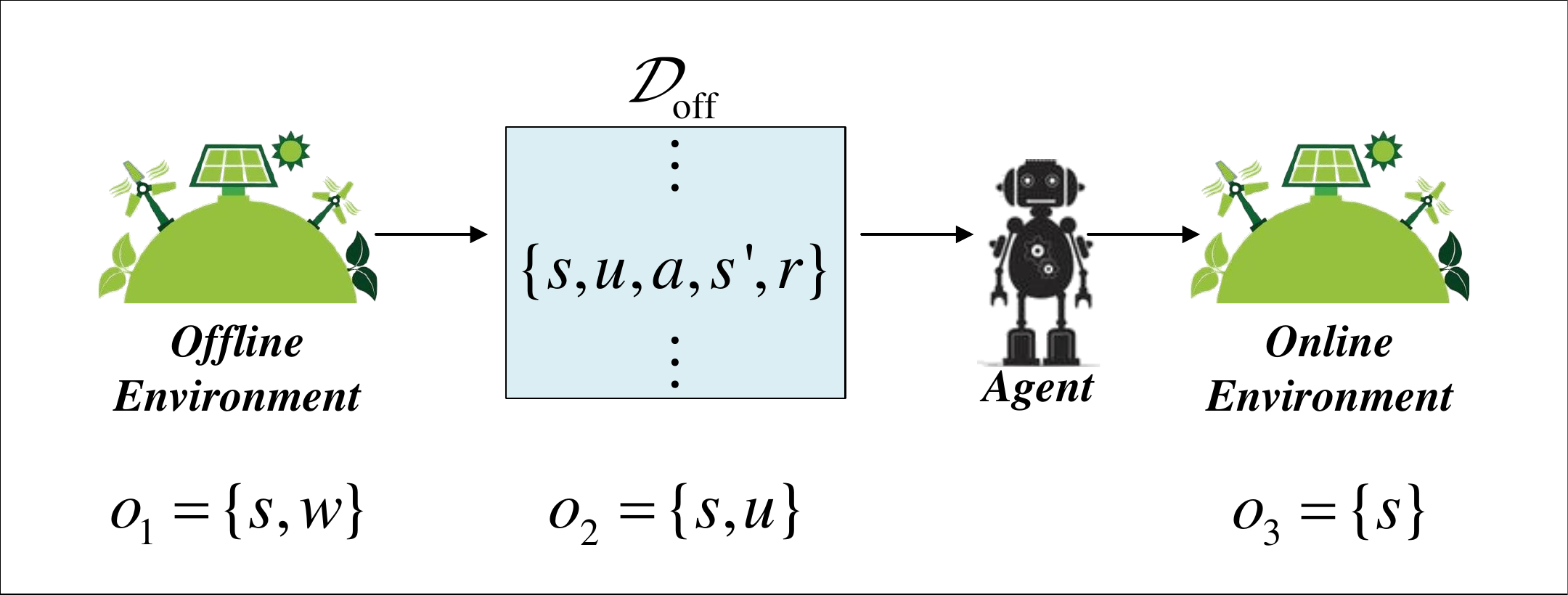}
    \caption{Diagram of the data flow framework, where the confounders in the offline data are partially observed.}
    \label{figure:dataflow_partially}
\end{figure}

\subsubsection{Reweighting Method}

Similar to Section~\ref{section:alg_for_unobserved}, we incorporate an importance sampling ratio into the loss function based on importance sampling in the reweighting method.  According to Proposition~\ref{proposition:cmdp1_loss2_holds}, the weighted loss function computed with offline data is unbiased, and the agent with this loss function can understand the environment correctly, while the original loss function as shown in Equation~\ref{eqn:loss1_cmdp1} is biased.
\begin{equation}
    \label{eqn:loss1_cmdp1}
    \begin{split}
    L_1\left(\phi,\mathcal{D}_{\offText}\right)&\triangleq\mathbb{E}_{s,a,s',r\sim{\mathcal{D}_{\offText}}}\left[f_\phi+h_\phi\right]\\
    &=\mathbb{E}_{s,a\sim{\mathcal{D}_{\offText}}}\left[\mathbb{E}_{s',r\sim{\hat{\PText}\left(\cdot,\cdot|s,a\right)}}\left[f_{\phi}+h_\phi\right]\right]\\
    \end{split}
\end{equation}

\begin{proposition}
    \label{proposition:cmdp1_loss2_holds}
    Under Assumption~\ref{assumption1} and the definitions of the CMDP and SCMs in Appendix~\ref{appendix:background_for_partially_observed_confounder}, it holds that
    \begin{equation}
    \label{eqn:loss2_cmdp1}
        \begin{split}
        L_2\left(\phi,\mathcal{D}_{\offText}\right)&\triangleq\mathbb{E}_{s,a\sim{\mathcal{D}_{\offText}}}\left[\mathbb{E}_{s',r\sim{\bar{\PText}\left(\cdot,\cdot|s,a\right)}}\left[f_\phi+h_\phi\right]\right]\\
        &=\mathbb{E}_{s,a\sim{\mathcal{D}_{\offText}}}\left[\mathbb{E}_{s',r,u\sim{\hat{\PText}\left(\cdot,\cdot,\cdot|s,a\right)}}\left[d_2\left(u,s,a\right)f_\phi+h_\phi\right]\right]\\
        &=\mathbb{E}_{s,a,s',r,u\sim{\mathcal{D}_{\offText}}}\left[d_2\left(u,s,a\right)f_\phi+h_\phi\right],\\
        \end{split}
    \end{equation}
    where $d_2\left(u,s,a\right)$ is defined as follows:
    \begin{equation}
    d_2\left(u,s,a\right)=\frac{\hat{\PText}\left(u|s\right)}{\hat{\PText}\left(u|s,a\right)}.
    \end{equation}
\end{proposition}

\begin{proof}
See Appendix~\ref{appendix:partially observed_reweight_loss_proof}.
\end{proof}

Similar to Section~\ref{section:alg_for_unobserved}, 
the new reweighting method can also be combined with existing model-free offline RL algorithms provided that Assumption~\ref{assumption:drlloss} and Assumption~\ref{assumption1} hold. 

\subsubsection{Resampling Method}
Similar to Section~\ref{section:alg_for_unobserved}, the resampling method extracts 5-tuples from the offline data with the unequal probabilities. However, due to the difference between SCMs, the probability of extracting elements from the offline data differs from Section~\ref{section:alg_for_unobserved}.
More specifically, the loss function for the offline RL algorithms with this resampling method is defined as follows:
\begin{equation}
    \begin{split}
    &L_3\left(\phi,\mathcal{D}_{\offText}\right)\triangleq\mathbb{E}_{I{\sim}p_2}\left[f_\phi\left(s_{I},a_{I},s_{I}',r_{I}\right)+h_\phi\left(s_{I},a_{I}\right)\right],\\
    \end{split}
\end{equation}
where $p_2\left(I=i\right)={d_{2,i}}/{\sum_{j=1}^{N}d_{2,j}}$, $d_{2,i}$ is a shorthand for ${\hat{\PText}\left(u_{i}|s_{i}\right)}/{\hat{\PText}\left(u_{i}|s_{i},a_{i}\right)}$, and $s_{i},a_{i},r_{i},s_{i}',u_{i}$ denotes the state, action, reward, next state and observable set of the confounder in the ith 5-tuple of the offline dataset ${D}_{off}$ respectively.
In this case, Proposition~\ref{proposition:l2eql3_cmdp1} holds.

\begin{proposition}
    \label{proposition:l2eql3_cmdp1}
    Under Assumption~\ref{assumption1} and the definitions of the CMDP and SCMs in Appendix~\ref{appendix:background_for_partially_observed_confounder}, the loss function of the resampling method is asymptotically equal to that of the reweighting method as in Equation~\ref{eqn:loss_resample_equal_cmdp1} provided that the dataset is large enough.
    \begin{equation}
        \label{eqn:loss_resample_equal_cmdp1}
        \begin{split}
        &\lim_{N\to\infty}\left(L_3\left(\phi,\mathcal{D}_{\offText}\right)-L_2\left(\phi,\mathcal{D}_{\offText}\right)\right)=0\\
        \end{split}
    \end{equation}
\end{proposition}

\begin{proof}
See Appendix~\ref{appendix:partially observed_resample_loss_proof}.
\end{proof}

According to Proposition~\ref{proposition:l2eql3_cmdp1}, the new resampling method can also deconfound the offline data if the offline dataset is large enough. Additionally, the resampling method can also be applied to existing model-free offline RL algorithms provided that Assumption~\ref{assumption:drlloss}
and Assumption~\ref{assumption1} hold. 

\section{Implementation Details}
\label{implementation_details}

Both the network architectures and optimizers are in the default settings of d3rlpy.
In the process of estimating $d_1\left(\tau\right)$ and $d_2\left(u,s,a\right)$, we need to add noises to all discrete random variables. 
Specifically, we choose a noise $\epsilon_{\theta,v}=\epsilon+\theta\left(B_v-0.5\right)$ where $\epsilon$ subjects to the uniform distribution $\mathcal{U}\left(-0.5,0.5\right)$, $B_v$ subjects to the beta distribution ${\BetaText\left(v,v\right)}$, $\theta=0.5$ and $v=5$.
The task settings of Figure~\ref{fig:unobserved_plot},\ref{fig:partially_observed_plot},\ref{fig:pendulum_v3_ablation_study.pdf},\ref{fig:pendulum_state_v3_ablation_study.pdf} are described in Table~\ref{table:env_para_in_results_section}.
Note that 1 epoch in Figure~\ref{fig:unobserved_plot},\ref{fig:partially_observed_plot},10-15 represents 3000 learning steps.
Other task settings and details of the four benchmark tasks are described in Appendix~\ref{additional_results}.

\begin{table}[htbp]
\vskip 0.15in
    \begin{center}
        \begin{small}
            \begin{sc}
                \begin{tabular}{l|r}
                    \toprule
                    \multicolumn{2}{l}{EmotionalPendulum}              \\
                    \midrule
                            $p_{fail}$&0.2\\
                        $v_T$&1\\
                        $I_{p,1}$&0.7\\
                         $odds_1$&4\\
                        \midrule
                     \multicolumn{2}{l}{WindyPendulum}           \\
                        \midrule
                        $p_{fail}$&0.1\\
                        $I_{p,2}$&0.9\\
                        $odds_2$&2.5\\
                        \midrule
                    \multicolumn{2}{l}{EmotionalPendulum*}         \\
                    \midrule
                    $v_T$&1       \\
                    $I_{p,1}$&0.7\\
                    $odds_1$&4\\
                        \midrule
                    \multicolumn{2}{l}{WindyPendulum*}              \\
                        \midrule
                     $I_{p,2}$&0.9     \\
                    $odds_2$&2.5\\
                    \bottomrule
                \end{tabular}
            \end{sc}
        \end{small}
    \end{center}
\vskip -0.1in
\caption{The task settings of Figure~\ref{fig:unobserved_plot},\ref{fig:partially_observed_plot},\ref{fig:pendulum_v3_ablation_study.pdf},\ref{fig:pendulum_state_v3_ablation_study.pdf}.}
\label{table:env_para_in_results_section}
\end{table}

\section{Simplification of the Causal Models}
\label{appendix:Simplification_of_the_Causal_Models}

In this section, CMDPs and SCMs for two specific cases of unobserved confounders are defined, and two corresponding simplified forms of $d_{1}\left(\tau\right)$ are derived.

\subsection{The Confounder Exists Only between the Action and Reward}

The CMDP for the first specific case of unobserved confounders can be denoted by a nine-tuple $\left\langle\mathcal{S},\mathcal{M},\mathcal{A},\mathcal{W},\mathcal{R},P_1,P_2,P_3,f_s,\mu_0\right\rangle$, where $\mathcal{S},\mathcal{M},\mathcal{A},\mathcal{W},\mathcal{R},P_2,P_3,\mu_0$ denote the same meanings as they have in the general CMDP for unobserved confounders in Section~\ref{background}, $P_1(r|s,w,m,s')$ denotes the reward distribution, and $f_s(s,m,\delta_s)$ denotes the next state transition function, in which $\delta_s$ denotes the error term in $f_s$.

The corresponding SCM in the offline setting can be defined as a four-tuple $\left\langle U, V, F, P_e\right\rangle$, where $U,V,P_e$ denote the same meanings as they have in Section~\ref{background}, and $F$ includes the reward distribution $P_1(r|s,w,m,s')$, the confounder transition distribution $P_2\left(w|s\right)$, the intermediate state transition distribution $P_3\left(m|s,a\right)$, the next state transition function $f_s(s,m,\delta_s)$, and the behavior policy $\pi_b\left(a|s,w\right)$.
The positivity assumption here is the same as the positivity assumption in Section~\ref{background}.

The corresponding SCM in the online setting can be defined as another four-tuple $\left(U, V, F, P_e\right)$, where $F$ includes the reward distribution $P_1(r|s,w,m,s')$, the confounder transition distribution $P_2\left(w|s\right)$, the intermediate state transition distribution $P_3\left(m|s,a\right)$, the next state transition function $f_s(s,m,\delta_s)$, and the policy $\pi\left(a|s\right)$.

Based on the CMDP and the corresponding SCMs, we can derive Proposition~\ref{proposition:sub1} and Proposition~\ref{proposition:l2eql3__cmdp2_sub1}.

\begin{proposition}
    \label{proposition:sub1}
    Under the definitions of the CMDP and SCMs in this subsection, it holds that
    \begin{equation}
        \label{eqn:loss2_sub1}
        \begin{split}
        L_2\left(\phi,\mathcal{D}_{\offText}\right)&\triangleq\mathbb{E}_{s,a\sim{\mathcal{D}_{\offText}}}\left[\mathbb{E}_{s',r\sim{\bar{\PText}\left(\cdot,\cdot|s,a\right)}}\left[f_\phi(s,a,s',r)+h_\phi(s,a)\right]\right]\\
        &=\mathbb{E}_{s,a\sim{\mathcal{D}_{\offText}},\delta_s\sim P_e}\left[\mathbb{E}_{r,m\sim{\hat{\PText}\left(\cdot,\cdot|s,a\right)}}\left[\bar{d}_1\left(\tau\right)f_\phi(s,a,f_s(s,m,\delta_s),r)+h_\phi(s,a)\right]\right]\\
        &=\mathbb{E}_{s,a,s',r,m\sim{\mathcal{D}_{\offText}}}\left[\bar{d}_1\left(\tau\right)f_\phi(s,a,s',r)+h_\phi(s,a)\right],\\
        \end{split}
    \end{equation}
    where $\bar{d}_1\left(\tau\right)$ is defined as follows:
    \begin{equation}
        \label{eqn:label_definition_sub1}
        \begin{split}
        \bar{d}_1\left(\tau\right)=\dfrac{\sum\limits_{a'}\hat{\PText}\left(r|m,a',s\right)\hat{\PText}\left(a'|s\right)}{\hat{\PText}\left(r|m,a,s\right)}.\\
        \end{split}
    \end{equation}
\end{proposition}

\begin{proof}
The proof of this proposition is similar to the proof of Proposition~\ref{proposition:1} in Appendix A, and thus is omitted.
\end{proof}

\begin{proposition}
    \label{proposition:l2eql3__cmdp2_sub1}
    Under the definitions of the CMDP and SCMs in this subsection, the loss function of the resampling method
    (i.e., Equation~\ref{eqn:loss_resample_unobserved_sub1})
    is asymptotically equal to that of the reweighting method as in Equation~\ref{eqn:loss_resample_equal_cmdp__2_sub1} provided that the dataset is large enough.
    \begin{equation}
        \label{eqn:loss_resample_unobserved_sub1}
        \begin{split}
        &L_3\left(\phi,\mathcal{D}_{\offText}\right)\triangleq\mathbb{E}_{I{\sim}\bar{p}_1}\left[f_\phi\left(s_{I},a_{I},s_{I}',r_{I}\right)+h_\phi\left(s_{I},a_{I}\right)\right],\\
        \end{split}
    \end{equation}
    where $\bar{p}_1\left(I=i\right)=\bar{d}_{1,i}/\sum_{j=1}^{N}\bar{d}_{1,j}$, and $\bar{d}_{1,i}$ is a shorthand for $\bar{d}_1\left(s_{i},m_{i},a_{i},s_{i}',r_{i}\right)$.
    \begin{equation}
        \label{eqn:loss_resample_equal_cmdp__2_sub1}
        \begin{split}
        &\lim_{N\to\infty}\left(L_3\left(\phi,\mathcal{D}_{\offText}\right)-L_2\left(\phi,\mathcal{D}_{\offText}\right)\right)=0\\
        \end{split}
    \end{equation}
\end{proposition}

\begin{proof}
The proof of this proposition is similar to the proof of Proposition~\ref{proposition:l2eql3__cmdp2} in Appendix A, and thus is omitted.
\end{proof}

\subsection{The Confounder Exists Only between the Action and Next State}
The CMDP for the second specific case of unobserved confounders can be denoted by a nine-tuple $\left\langle\mathcal{S},\mathcal{M},\mathcal{A},\mathcal{W},\mathcal{R},P_1,P_2,P_3,f_r,\mu_0\right\rangle$, where $\mathcal{S},\mathcal{M},\mathcal{A},\mathcal{W},\mathcal{R},P_2,P_3,\mu_0$ denote the same meanings as they have in the general CMDP for unobserved confounders in Section~\ref{background}, $P_1(s'|s,w,m)$ denotes the next state transition distribution, and $f_r(s,m,\delta_r)$ denotes the reward function, in which $\delta_r$ denotes the error term in $f_r$.

The corresponding SCM in the offline setting can be defined as a four-tuple $\left\langle U, V, F, P_e\right\rangle$, where $U,V,P_e$ denote the same meanings as they have in Section~\ref{background}, and $F$ includes the next state transition distribution $P_1(s'|s,w,m)$, the confounder transition distribution $P_2\left(w|s\right)$, the intermediate state transition distribution $P_3\left(m|s,a\right)$, the reward function $f_r(s,m,\delta_r)$, and the behavior policy $\pi_b\left(a|s,w\right)$.
The positivity assumption here is the same as the positivity assumption in Section~\ref{background}.

The corresponding SCM in the online setting can be defined as another four-tuple $\left(U, V, F, P_e\right)$, where $F$ includes the next state transition distribution $P_1(s'|s,w,m)$, the confounder transition distribution $P_2\left(w|s\right)$, the intermediate state transition distribution $P_3\left(m|s,a\right)$, the reward function $f_r(s,m,\delta_r)$, and the policy $\pi\left(a|s\right)$.

Based on the CMDP and the corresponding SCMs, we can derive Proposition~\ref{proposition:sub2} and Proposition~\ref{proposition:l2eql3__cmdp2_sub2}.

\begin{proposition}
    \label{proposition:sub2}
    Under the definitions of the CMDP and SCMs in this subsection, it holds that
    \begin{equation}
        \label{eqn:loss2_sub2}
        \begin{split}
        L_2\left(\phi,\mathcal{D}_{\offText}\right)&\triangleq\mathbb{E}_{s,a\sim{\mathcal{D}_{\offText}}}\left[\mathbb{E}_{s',r\sim{\bar{\PText}\left(\cdot,\cdot|s,a\right)}}\left[f_\phi(s,a,s',r)+h_\phi(s,a)\right]\right]\\
        &=\mathbb{E}_{s,a\sim{\mathcal{D}_{\offText}},\delta_r\sim P_e}\left[\mathbb{E}_{s',m\sim{\hat{\PText}\left(\cdot,\cdot|s,a\right)}}\left[\tilde{d}_1\left(\tau\right)f_\phi(s,a,s',f_r(s,m,\delta_r))+h_\phi(s,a)\right]\right]\\
        &=\mathbb{E}_{s,a,s',r,m\sim{\mathcal{D}_{\offText}}}\left[\tilde{d}_1\left(\tau\right)f_\phi(s,a,s',r)+h_\phi(s,a)\right],\\
        \end{split}
    \end{equation}
    where $\tilde{d}_1\left(\tau\right)$ is defined as follows:
    \begin{equation}
        \label{eqn:label_definition_sub2}
        \begin{split}
        \tilde{d}_1\left(\tau\right)=\dfrac{\sum\limits_{a'}\hat{\PText}\left(s'|m,a',s\right)\hat{\PText}\left(a'|s\right)}{\hat{\PText}\left(s'|m,a,s\right)}.\\
        \end{split}
    \end{equation}
\end{proposition}

\begin{proof}
The proof of this proposition is similar to the proof of Proposition~\ref{proposition:1} in Appendix A, and thus is omitted.
\end{proof}

\begin{proposition}
    \label{proposition:l2eql3__cmdp2_sub2}
    Under the definitions of the CMDP and SCMs in this subsection, the loss function of the resampling method
    (i.e., Equation~\ref{eqn:loss_resample_unobserved_sub2})
    is asymptotically equal to that of the reweighting method as in Equation~\ref{eqn:loss_resample_equal_cmdp__2_sub2} provided that the dataset is large enough.    
    \begin{equation}
        \label{eqn:loss_resample_unobserved_sub2}
        \begin{split}
        &L_3\left(\phi,\mathcal{D}_{\offText}\right)\triangleq\mathbb{E}_{I{\sim}\tilde{p}_1}\left[f_\phi\left(s_{I},a_{I},s_{I}',r_{I}\right)+h_\phi\left(s_{I},a_{I}\right)\right],\\
        \end{split}
    \end{equation}
    where $\tilde{p}_1\left(I=i\right)=\tilde{d}_{1,i}/\sum_{j=1}^{N}\tilde{d}_{1,j}$, and $\tilde{d}_{1,i}$ is a shorthand for $\tilde{d}_1\left(s_{i},m_{i},a_{i},s_{i}',r_{i}\right)$.
    \begin{equation}
        \label{eqn:loss_resample_equal_cmdp__2_sub2}
        \begin{split}
        &\lim_{N\to\infty}\left(L_3\left(\phi,\mathcal{D}_{\offText}\right)-L_2\left(\phi,\mathcal{D}_{\offText}\right)\right)=0\\
        \end{split}
    \end{equation}
\end{proposition}

\begin{proof}
The proof of this proposition is similar to the proof of Proposition~\ref{proposition:l2eql3__cmdp2} in Appendix A, and thus is omitted.
\end{proof}

\section{Complete Empirical Results}
\label{additional_results}

The task settings of EmotionalPendulum, WindyPendulum, EmotionalPendulum* and WindyPendulum* are shown in the following two subsections.
The additional empirical results of WindyPendulum, where the confounders are unobserved in the offline data, verify the effectiveness of the deconfounding methods proposed in Section~\ref{section:alg_for_unobserved}. In contrast, the additional empirical results of WindyPendulum*, where the confounders are partially observed in the offline data, verify the effectiveness of the deconfounding methods proposed in Appendix~B.
In each subsection, we first describe the causal models of the offline data in the four benchmark tasks, and then compare BC, the original offline RL algorithms and the offline RL algorithms combined with our deconfounding methods in different task settings. 
The task settings follow Pendulum in OpenAI Gym except for the settings described in each subsection.

\subsection{Unobserved Confounders}

The causal model in the offline setting, which corresponds to EmotionalPendulum for unobserved confounders, is described as follows. 
In this causal model, $w_1 \in \{True,False\}$ represents whether the human sitting in the free end of the pendulum has negative emotions, $w_2 \in \{True, False\}$ represents whether the human has negative expressions, $s=(x,y,v)$ represents the Cartesian coordinates $(x,y)$ and angular velocity $(v)$ of the free end of the pendulum,
the action $a \in \{-2,-1,0,1,2\}$ represents the torque that the human wants to apply to the free end of the pendulum,
the intermediate state $m \in \{-2,-1,0,1,2\}$ represents the actual torque applied to the free end,
and $r=r_o+r_a$ represents the reward where $r_o$ represents the original reward in Pendulum of OpenAI Gym, and $r_a$ represents the additional reward which is used to encourage the human with negative expressions. Note that the coordinates $(x,y)$ are measured by a sensor which vibrates slightly at a distance of 1m from the fixed end of the pendulum, i.e., $x=\cos(\theta) \times l,y=\sin(\theta) \times l$ where $\theta$ denotes the angle in radians and $l \in (0.5,1.5)$ denotes the truncated normal distribution with the mean $\mu = 1$ and the variance $\sigma = 0.01$.

The offline data generation process of this causal model is described as follows. The odds that the human does not have negative emotions are $odds_1=\PrText(w_1=False)/\PrText(w_1=True)$.
If the human has negative emotions, he/she is most likely to have negative expressions, i.e., 
$\PrText(w_2=True|w_1=True)=0.99,\PrText(w_2=False|w_1=True)=0.01,\PrText(w_2=True|w_1=False)=0.01,\PrText(w_2=False|w_1=False)=0.99$.
An agent is trained using the SAC algorithm in an environment where $w_1$, $w_2$ and $s$ are observable for 300000 learning steps.
This agent will be used to simulate the human in a rational state, and give the rational action $\bar{a}$ at each step. We assume that the human may feel afraid and decide to slow down (i.e., $a=\tilde{a}$) if the speed is too fast (i.e., if the speed is above the threshold $|v|>v_T$), and the human may feel boring and decide to speed up (i.e., $a=\hat{a}$) if the speed is too slow (i.e., if $|v| \leq v_T$). Specifically, 
$\PrText(a=\bar{a}|w_1=False) =1,
\PrText(a=\bar{a}|w_1=True,v=0)=1,
\PrText(a=\hat{a}|w_1=True,|v| \leq v_T,v \neq 0)=I_{p,1},
\PrText(a=\bar{a}|w_1=True,|v| \leq v_T,v \neq 0)=1-I_{p,1},
\PrText(a=\tilde{a}|w_1=True,|v|>v_T)=I_{p,1},
\PrText(a=\bar{a}|w_1=True,|v|>v_T)=1-I_{p,1}$, where $I_{p,1}$ denotes the probability of the human choosing irrational actions when he/she has negative emotions. 
The action $a \in \{-2,-1,0,1,2\}$ influences the next state $s' \in \mathcal{S}$ through the intermediate state $m \in \{-2,-1,0,1,2\}$.
In most cases, $m$ is equal to $a$
(i.e., $\PrText(m=a)=1-p_{fail}+p_{fail}/5$.).
However, in some cases, the human/agent fails to control the pendulum and $m$ randomly chooses an action
(i.e., $\forall\bar{m} \neq a,\PrText(m=\bar{m})=p_{fail}/5$.).
The environment will return an extra reward to encourage the human if his/her expressions are negative (i.e., $r_a$ is subject to $\NText\left(10,1\right)$ if $w_2=True$, and is subject to $\NText\left(0,1\right)$ if $w_2=False$).
The transition function in EmotionalPendulum is $\theta',v'=g(m)$,
where $\theta'$ denotes the angle of the next state, $v'$ denotes the angular velocity of the next state, and $g$ denotes the original transition function in Pendulum of OpenAI Gym. 
The human interacts with the environment for 100000 steps to generate offline data.

The causal model in the offline setting, which corresponds to WindyPendulum for unobserved confounders,
is described as follows.
In this causal model, $w_2 \in \{0,1,2\}$ represents the direction of the wind, where $0$ represents the wind from right to left, $1$ represents no wind and $2$ represents the wind from left to right, $w_1 \in \{True,False\}$ represents whether the human is afraid because of the wind, $s$, $a$ and $m$ represent the same meanings as $s$, $a$ and $m$ in EmotionalPendulum,
and $r=r_o+r_a$ represents the reward, where $r_o$ represents the original reward in Pendulum of OpenAI Gym, and $r_a$ represents the additional reward used to encourage the human if a strong wind exists.

The offline data generation process of this causal model is described in the following. The odds of no wind are 
$odds_2=\PrText(w_2=1)/(\PrText(w_2=0)+\PrText(w_2=2))$.
If there is a gust of wind,
the human will get scared, i.e., 
$\PrText(w_1=True|w_2 \neq 1)=1$.
We use the SAC algorithm to train an agent in an environment where $w_2$ and $s$ are observable for 300000 learning steps.
This agent will be used to simulate the human in the rational state, and give the rational action $\bar{a}$ at each step. We assume that the human in a state of fear may decide to slow down (i.e., $a=\tilde{a}$), and may choose a force opposite to the component of the wind force along the tangent (i.e., $a=\hat{a}$). Specifically, 
$\PrText(a=\bar{a}|w_1=False)=1, 
\PrText(a=\bar{a}|w_1=True)=1-I_{p,2},
\PrText(a=\hat{a}|w_1=True)=I_{p,2}/2,
\PrText(a=\tilde{a}|w_1=True)=I_{p,2}/2$
, where $I_{p,2}$ denotes the probability that the human in a state of fear
chooses irrational actions. 
The causal mechanism that generates $m$ 
is the same as that in EmotionalPendulum.
The environment will return an extra reward to encourage the human if there is a strong wind, i.e., $r_a$ is subject to N(10,1) if $w_2 \neq 1$, and is subject to N(0,1) if $w_2=1$.
The transition function in WindyPendulum is $\theta',v'=g\left(m-f_w \times \cos(\theta)\right)$,
where $\theta'$, $v'$ and $g$ denote the same meanings as $\theta'$, $v'$ and $g$ in EmotionalPendulum
and $f_w$ denotes the wind force as shown in Equation~\ref{eqn:thewindforce}. 
Offline data are generated from 100000 steps of the human interacting with the environment.
\begin{equation}
    \label{eqn:thewindforce}
    f_w=\left\{
    \begin{aligned}
    -5 &      & w_2=0 \\
    0 &      & w_2=1 \\
    5 &      & w_2=2
    \end{aligned}
    \right.
\end{equation}

As shown in Figure~\ref{fig:pendulum_state_v3_0914109.pdf} and Figure~\ref{fig:pendulum_state_v3_09141085.pdf},
the offline RL algorithms combined with our deconfounding methods learn faster than the original offline RL algorithms under different settings of WindyPendulum for unobserved confounders, which verifies the robustness of our deconfounding methods.
Since the human may adopt irrational actions in our tasks, his/her behavior policy is not good, and thus BC performs poorly in our tasks. The empirical results in Figure~\ref{fig:pendulum_state_v3_0914109.pdf} and Figure~\ref{fig:pendulum_state_v3_09141085.pdf} show that BC performs worse than the offline RL algorithms combined with our deconfounding methods.

\paragraph{Ablation study} 
Figure~\ref{fig:pendulum_v3_ablation_study.pdf} and Figure~\ref{fig:pendulum_state_v3_ablation_study.pdf} provide an ablation study of the center sampling method $k$-means~\cite{macqueen1967classification,lloyd1982least}.
In the process of estimating the conditional probability densities using LSCDE, we need to select center points for the kernel functions.
DQN\_RW and DQN\_RS denote the deconfounding RL algorithms that use $k$-means clustering to determine $k$ kernel centers. DQN\_RW* and DQN\_RS* denote the deconfounding RL algorithms that randomly select $k$ points as kernel centers. Other algorithms are denoted similarly.
Clearly, the deconfounding RL algorithms using the center sampling method $k$-means perform better than or similarly to the deconfounding RL algorithms that randomly select $k$ points as kernel centers.

\subsection{Partially Observed Confounders}

\begin{figure}[htbp]
    \centering
    \subfigure{
        \includegraphics[width=1.55in]{figures/12.pdf}
        \label{figure:Friction1}
    }
    \subfigure{
	\includegraphics[width=1.55in]{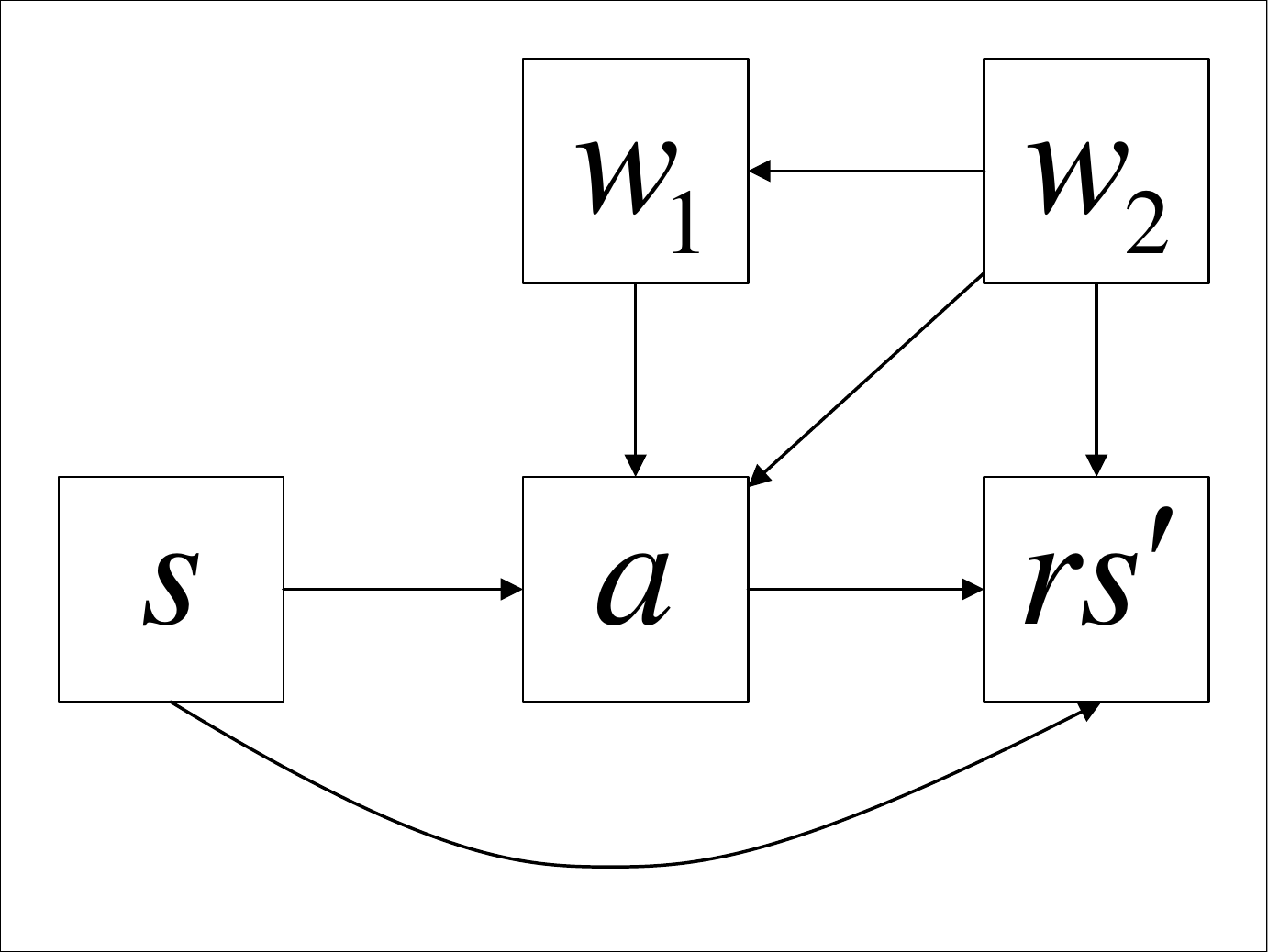}
	\label{figure:Friction2}
    }
    \\
    \subfigure{
        \includegraphics[width=1.55in]{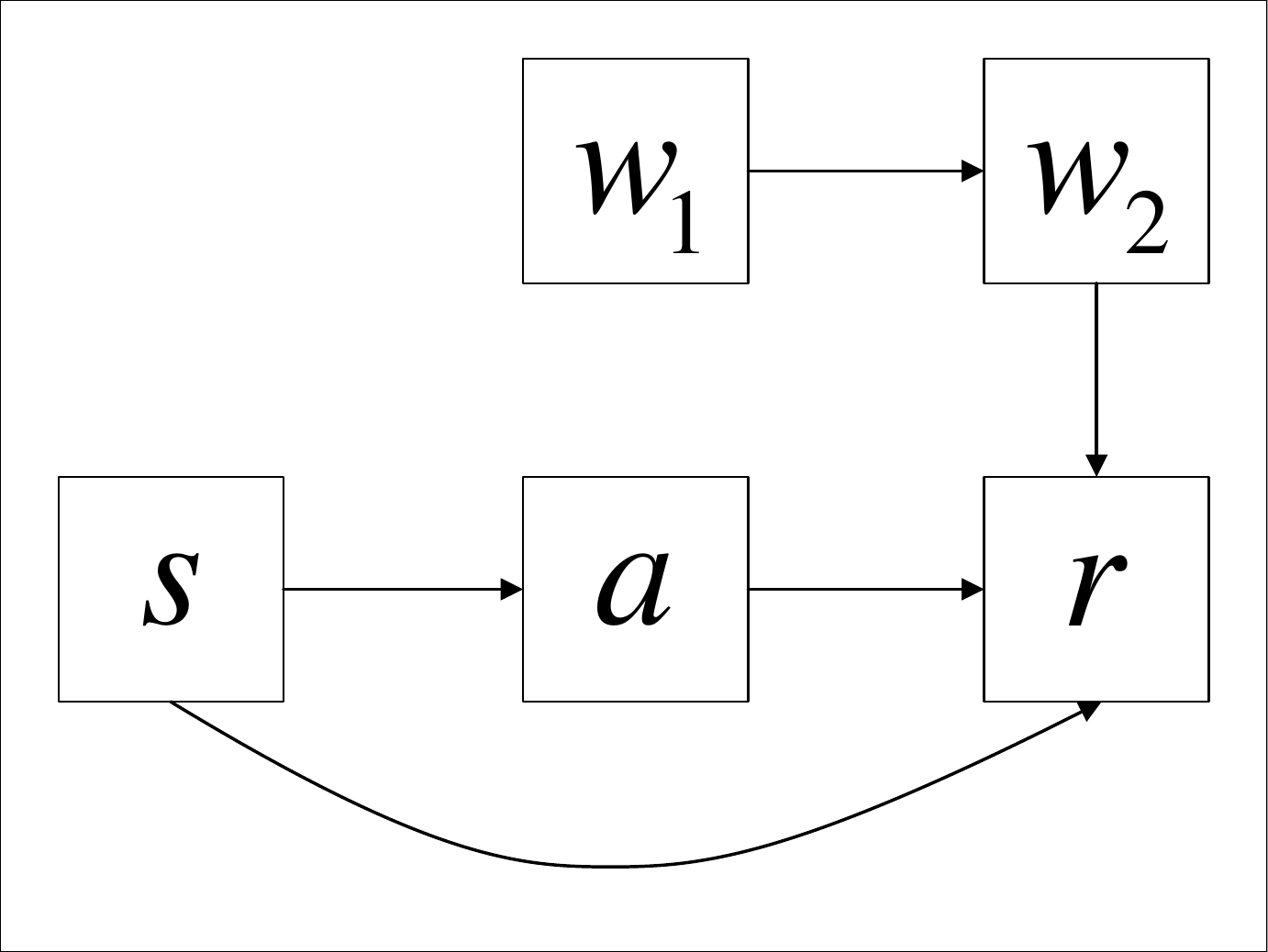}
        \label{11091}
    }
    \subfigure{
	\includegraphics[width=1.55in]{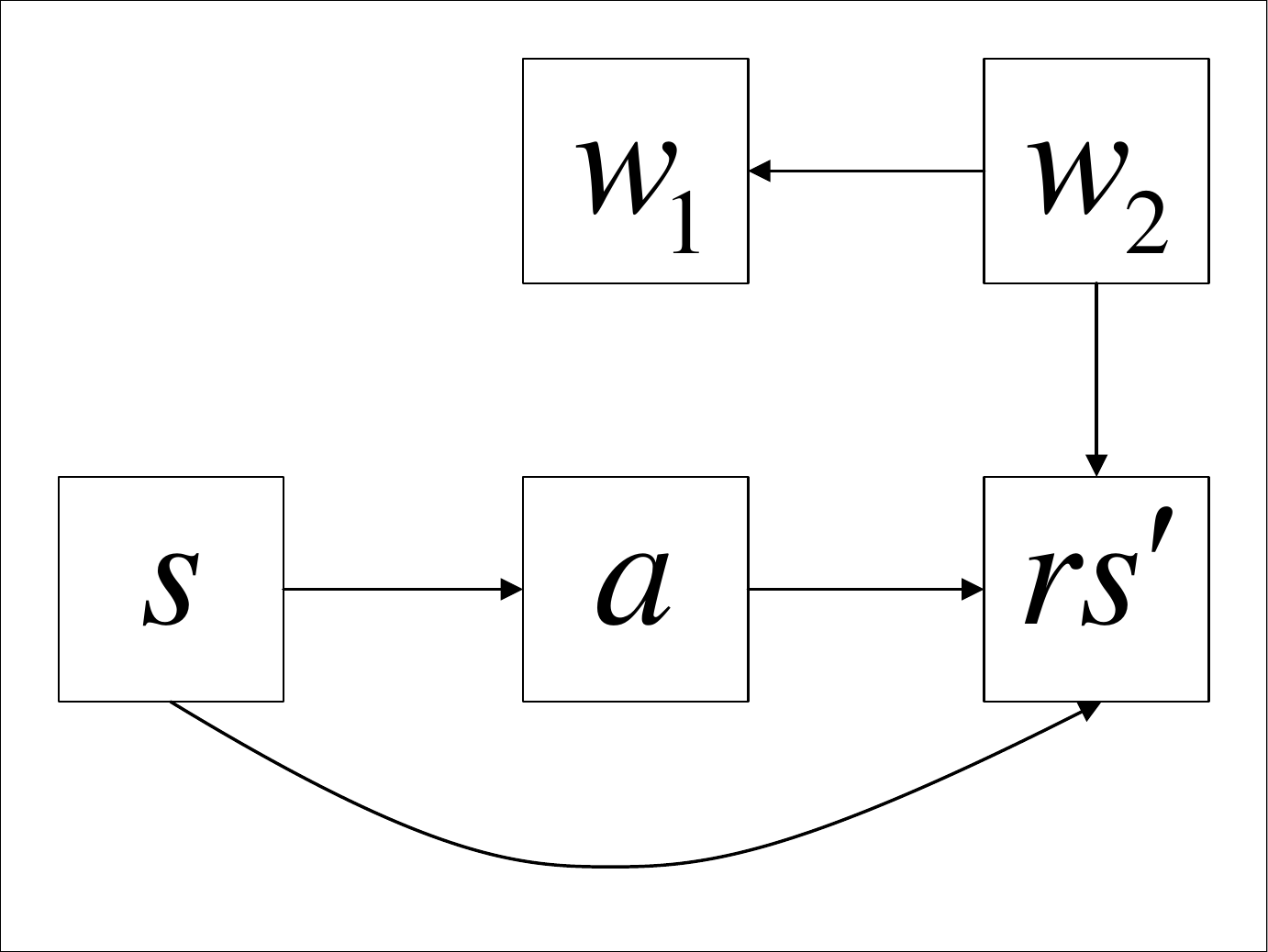}
        \label{123288}
    }
    \caption{
        \textbf{Left}: the causal graphs depicting the offline and online data generating processes in EmotionalPendulum*. 
        \textbf{Right}: the causal graphs depicting the offline and online data generating processes in WindyPendulum*. 
        In both tasks, the confounders in the offline data are partially observable.
    }
    \label{figure:scm_partially_observed}
\end{figure}

The causal graphs of EmotionalPendulum* and WindyPendulum* for partially observed confounders are shown in Figure~\ref{figure:scm_partially_observed}. It is obvious that these causal graphs are special cases of the causal graphs described in Appendix~\ref{appendix:background_for_partially_observed_confounder}.
The only difference between the causal models 
of the tasks for unobserved confounders and partially observed confounders is that
there is no intermediate state $m$ intercepting every directed path from $a$ to $s'$ anymore. 
In other words, the controller of the pendulum never fails, i.e., the action that the human wants to take is equal to the actual action executed by the machine.
Meanwhile, the transition functions are changed correspondingly. In EmotionalPendulum*, $\theta',v'=g(a)$. In WindyPendulum*, $\theta',v'=g\left(a-f_w \times \cos(\theta)\right)$.

Similar to EmotionalPendulum and WindyPendulum, in EmotionalPendulum* and WindyPendulum*,
some sensors are used to collect the confounded offline data generated in the human-environment interaction process. The environmental information $o_2$ in these offline data includes $s$ and a subset $u$ of the confounder. In EmotionalPendulum*, $u=\{w_2\}$ where $w_2$ represents whether the human has negative expressions. In WindyPendulum*, $u=\{w_2\}$ where $w_2$ represents the direction of the wind.
Except for what we have described above, the tasks for unobserved confounders and partially observed confounders are the same.

The empirical results in Figure~\ref{fig:pendulum_state_v1_151_0.85.pdf} and Figure~\ref{fig:pendulum_state_v1_161_0.85.pdf}
demonstrate that the offline RL algorithms combined with our deconfounding methods learn faster than the original offline RL algorithms under different settings of WindyPendulum* for partially observed confounders, and thus verify the robustness of our deconfounding methods.
As shown in Figure~\ref{fig:pendulum_state_v1_151_0.85.pdf} and Figure~\ref{fig:pendulum_state_v1_161_0.85.pdf}, BC performs worse than the offline RL algorithms combined with our deconfounding methods.

\begin{figure}[ht]
  \centering
  \includegraphics[width=\textwidth]{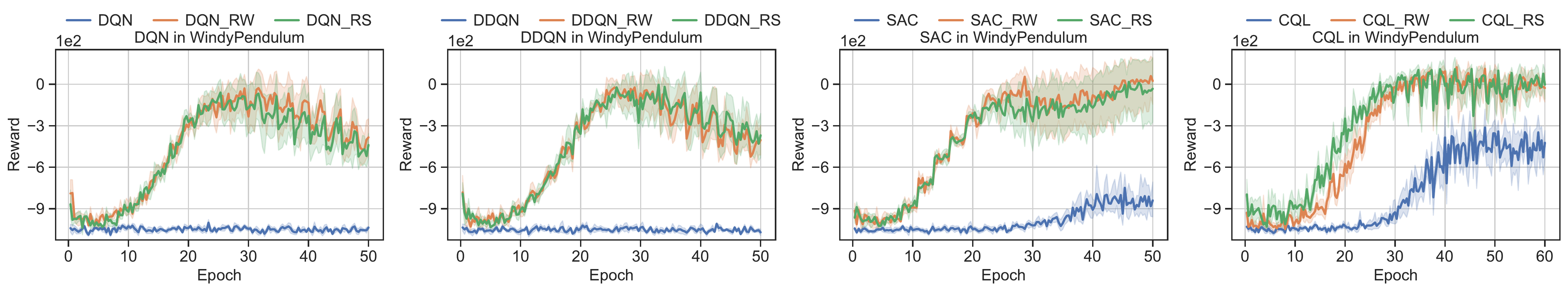}
  \caption{Performance of the deep RL algorithms with and without our deconfounding methods in WindyPendulum. The optimal average reward of BC over 50 epochs is -987.4. The rewards are tested over 20 episodes every 1000 learning steps. We divide the 50 epochs into 10 parts, and calculate the average reward of BC in each part. Finally, we take the largest average reward in all parts as the optimal average reward of BC. The task settings of this figure are as follows: $p_{fail}=0.1, I_{p,2}=0.9, odds_2=2$.}
  \label{fig:pendulum_state_v3_0914109.pdf}
\end{figure}

\begin{figure}[ht]
  \centering
  \includegraphics[width=\textwidth]{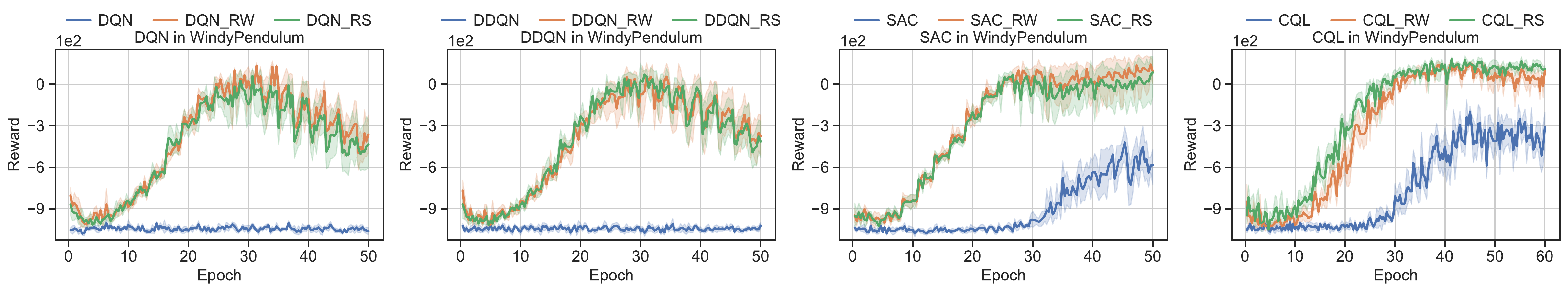}
  \caption{Performance of the deep RL algorithms with and without our deconfounding methods in WindyPendulum. The optimal average reward of BC over 50 epochs is -916.1. The task settings of this figure are as follows: $p_{fail}=0.1, I_{p,2}=0.85, odds_2=2$.}
  \label{fig:pendulum_state_v3_09141085.pdf}
\end{figure}


\begin{figure}[ht]
  \centering
  \includegraphics[width=\textwidth]{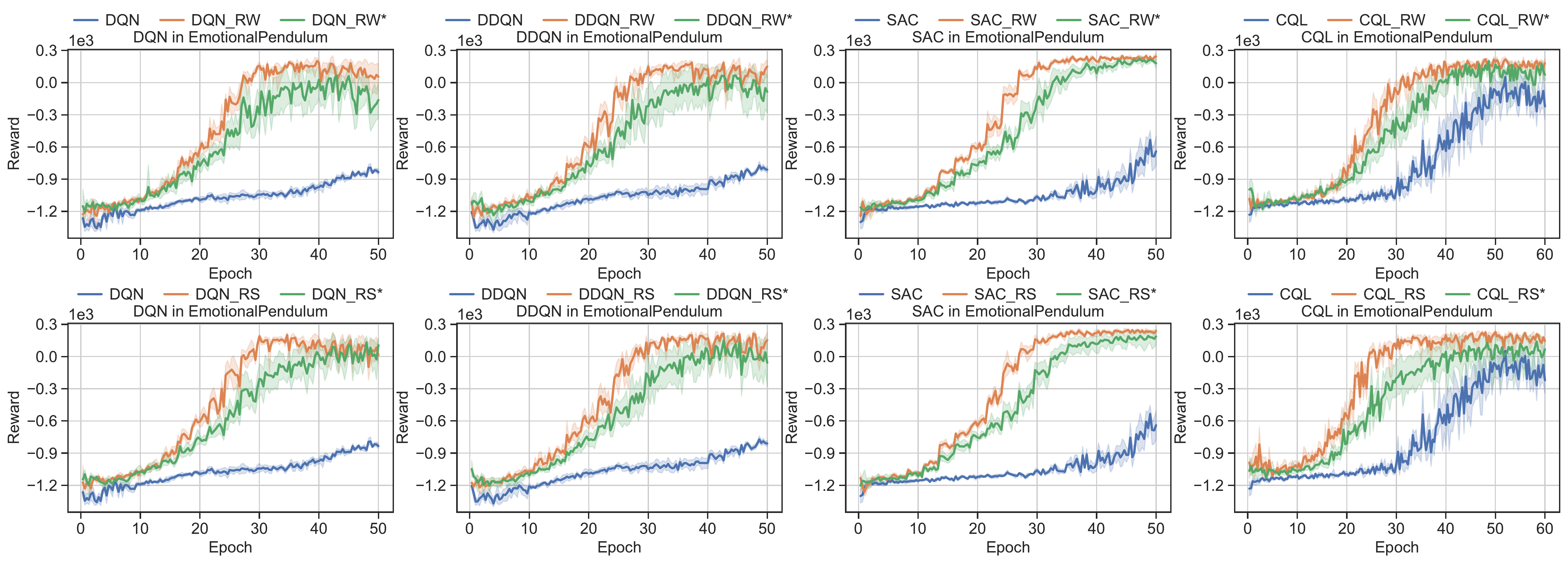}
  \caption{Performance of the deconfounding RL algorithms that use $k$-means clustering to determine $k$ kernel centers and the deconfounding RL algorithms that randomly select $k$ points as kernel centers in EmotionalPendulum. The task settings of this figure are described in Appendix~C.}
  \label{fig:pendulum_v3_ablation_study.pdf}
\end{figure}

\begin{figure}[ht]
  \centering
  \includegraphics[width=\textwidth]{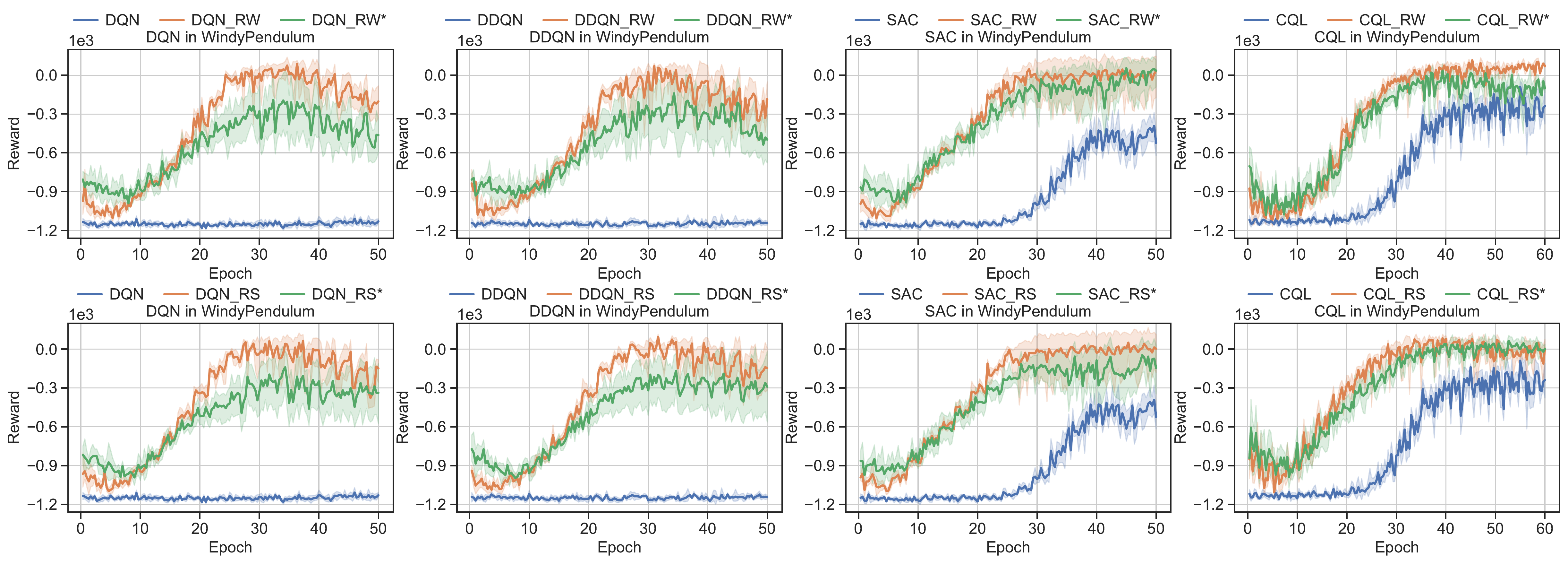}
  \caption{Performance of the deconfounding RL algorithms that use $k$-means clustering to determine $k$ kernel centers and the deconfounding RL algorithms that randomly select $k$ points as kernel centers in WindyPendulum. The task settings of this figure are described in Appendix~C.}
  \label{fig:pendulum_state_v3_ablation_study.pdf}
\end{figure}

\begin{figure}[ht]
  \centering
  \includegraphics[width=\textwidth]{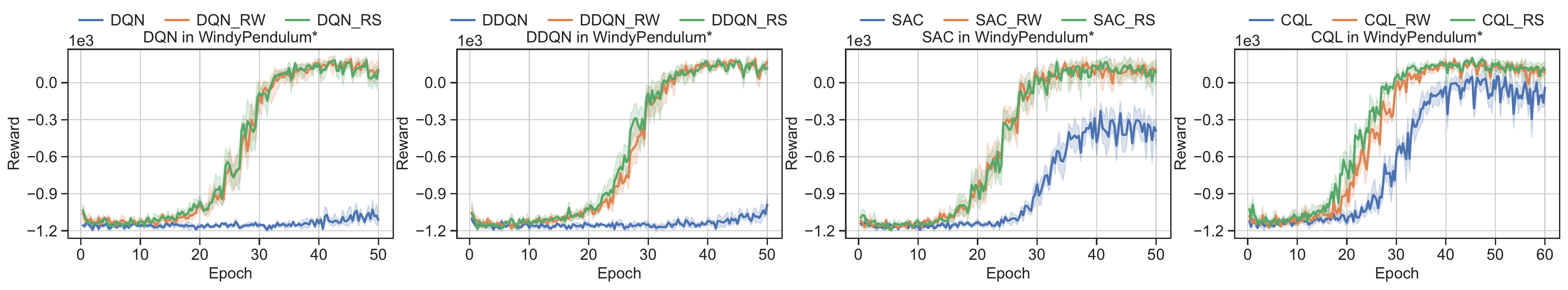}
  \caption{Performance of the deep RL algorithms with and without our deconfounding methods in WindyPendulum*. The optimal average reward of BC over 50 epochs is -826.5. The task settings of this figure are as follows: $I_{p,2}=0.85, odds_2=2.5$.}
  \label{fig:pendulum_state_v1_151_0.85.pdf}
\end{figure}

\begin{figure}[ht]
  \centering
  \includegraphics[width=\textwidth]{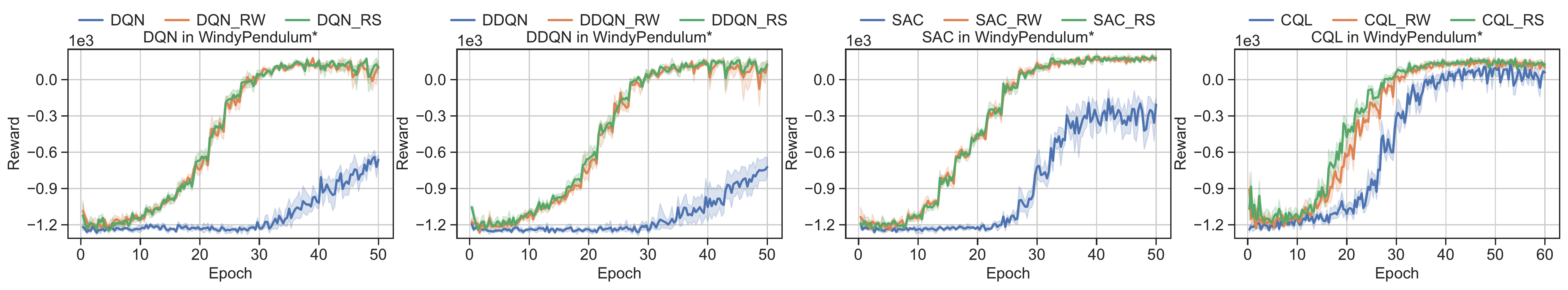}
  \caption{Performance of the deep RL algorithms with and without our deconfounding methods in WindyPendulum*. The optimal average reward of BC over 50 epochs is -432.9. The task settings of this figure are as follows: $I_{p,2}=0.85, odds_2=3$.}
  \label{fig:pendulum_state_v1_161_0.85.pdf}
\end{figure}

\end{document}